\documentclass{article}

\PassOptionsToPackage{numbers, compress}{natbib}

    \usepackage[final]{neurips_2024}

\usepackage[utf8]{inputenc} 
\usepackage[T1]{fontenc}    
\usepackage{hyperref}       
\usepackage{url}            
\usepackage{booktabs}       
\usepackage{amsfonts}       
\usepackage{nicefrac}       
\usepackage{microtype}      
\usepackage{xcolor}         
\usepackage{amsmath,amsthm} 
\usepackage{graphicx}
\usepackage{subfigure}
\usepackage{enumitem}
\setlist[itemize]{leftmargin=1cm}
\setlist[enumerate]{leftmargin=1cm}
\usepackage{enumitem}

\usepackage{algorithm}
\usepackage{algorithmic}

\usepackage{multirow}

\newtheorem{theorem}{Theorem}[section]

\newtheorem{lemma}[theorem]{Lemma}

\theoremstyle{definition}

\theoremstyle{remark}
\newtheorem{remark}[theorem]{Remark}

\newtheoremstyle{named}{}{}{\itshape}{}{\bfseries}{.}{.5em}{\thmnote{#3 }#1}
\theoremstyle{named}

\usepackage{comment}

\title{Transfer Learning for Diffusion Models}

\author{
  Yidong Ouyang$^{1}$, Liyan Xie$^{2}$\thanks{Correspondence to: Liyan Xie, \texttt{liyanxie@umn.edu}. 
  }, Hongyuan Zha$^{3}$, Guang Cheng$^{1}$ \\
  \\
  $^{1}$Department of Statistics and Data Science, University of California, Los Angeles \\
  $^{2}$Department of Industrial and Systems Engineering, University of Minnesota Twin Cities \\
  $^{3}$School of Data Science, Chinese University of Hong Kong, Shenzhen \\
   {\footnotesize	 
\texttt{yidongouyang@g.ucla.edu}; \texttt{liyanxie@umn.edu}; \texttt{zhahy@cuhk.edu.cn}; \texttt{guangcheng@ucla.edu}}
}

\begin{document}

\maketitle

\begin{abstract}
Diffusion models, a specific type of generative model, have achieved unprecedented performance in recent years and consistently produce high-quality synthetic samples. A critical prerequisite for their notable success lies in the presence of a substantial number of training samples, which can be impractical in real-world applications due to high collection costs or associated risks. Consequently, various finetuning and regularization approaches have been proposed to transfer knowledge from existing pre-trained models to specific target domains with limited data. This paper introduces the Transfer Guided Diffusion Process (TGDP), a novel approach distinct from conventional finetuning and regularization methods. 
We prove that the optimal diffusion model for the target domain integrates pre-trained diffusion models on the source domain with additional guidance from a domain classifier. 
We further extend TGDP to a conditional version for modeling the joint distribution of data and its corresponding labels, together with two additional regularization terms to enhance the model performance. We validate the effectiveness of TGDP on both simulated and real-world datasets.
\end{abstract}


\section{Introduction}\label{sec:into}

Diffusion models have achieved remarkable success in modeling data distributions and generating various types of synthetic data, such as images \citep{Ho2020DDPM,Song2021ScoreBasedGM, Karras2022ElucidatingTD}, videos \citep{Ho2022VideoDM}, vision language \citep{Ruiz2022DreamBoothFT, Saharia2022PhotorealisticTD, Ramesh2022HierarchicalTI}, and time series \citep{Tashiro2021CSDICS}. 
However, their success heavily relies on the availability of a large number of training samples. In real-world applications, acquiring ample samples for specific tasks can be challenging due to the high costs associated with data collection or labeling, or the potential risks involved. 
Therefore, an important research question is how to effectively transfer knowledge from a pre-trained generative model in the source domain (using existing large-scale datasets) to a target domain (for specific tasks) where data is limited.

Training a generative model directly or finetuning a pre-trained generative model on limited data from the target domain often results in significant performance degradation due to overfitting and memorization. To address these issues, numerous studies have proposed methods in generative domain adaptation, including the GAN-based models \citep{Yang2023ImageSU, Yang2021OneShotGD, Zhang2022GeneralizedOD, Alanov2022HyperDomainNetUD, Zhao2022ACL,Ojha2021FewshotIG, Zhang2022TowardsDA, Duan2023WeditGANFI, Hou2022DynamicWS, Xiao2022FewSG, Li2020FewshotIG, Zhao2020OnLP}, diffusion-based model \citep{Moon2022FinetuningDM, Zhu2022FewshotIG, Xie2023DiffFitUT}, etc. Specifically, approaches using diffusion models can be divided into two categories: finetuning lightweight adapters \citep{Moon2022FinetuningDM, Xie2023DiffFitUT} and finetuning with regularization \citep{Zhu2022FewshotIG}.
%
%
Approaches involving finetuning lightweight adapters focus on adjusting only a subset of parameters in a pre-trained model. 
The primary challenge here is identifying which parameters to finetune. 
This process is typically heuristic and requires preliminary experiments to identify the most efficient parameters for adjustment. Additionally, the specific parameters to be finetuned can vary across different neural network architectures. 
On the other hand, the challenge in incorporating regularization during the finetuning process is the heuristic design of the regularization term, which can significantly alter the optimization landscape. We refer to Appendix \ref{app:related} for a more detailed discussion of existing literature.

In this work, we introduce a new approach, termed Transfer Guided Diffusion Process (TGDP), to transfer knowledge in the source domain generative model to the target domain with limited samples. Unlike {finetuning-based methods that primarily use the pre-trained model as an initialization point, TGDP leverages the pre-trained model as a plug-and-play prior.} We show that the score function for the diffusion model on the target domain is the score function on the source domain (which can be pre-trained) with additional guidance as shown by Theorem \ref{thm:IS_guidance} and Theorem \ref{thm:IS_guidance_conditional}. The guidance network is related to the density ratio of the target and source domain data  distributions. Consequently, we convert the original optimization problem for a diffusion model on the target domain into estimating the density ratio.

We utilize a domain classifier (binary classifier) along with samples from both domains to efficiently estimate the density ratio. 
Furthermore, we introduce two additional regularization terms for better training and calibration of the guidance network. These regularization terms are equivalent forms that the optimal guidance network should satisfy, ensuring they do not alter the original optimization problem.
%
We validate the effectiveness of our approach through experiments on Gaussian mixture simulations and real electrocardiogram (ECG) data. Under both fidelity and utility evaluation criteria, TGDP consistently outperforms finetuning-based methods. 

Our contributions can be summarized as follows. 

\begin{itemize}[leftmargin=1.5em]
    \item  We introduce a new framework, the Transfer Guided Diffusion Process (TGDP), for transferring a pre-trained diffusion model from the source domain to the target domain.
    \item We extend TGDP to a conditional version for modeling the joint distribution of data and its corresponding labels, along with two additional regularization terms, which are important for practical applications and downstream tasks.   
    \item TGDP demonstrates superior performance over finetuning-based methods on Gaussian mixture simulations and on benchmark electrocardiogram (ECG) data.
\end{itemize}

The rest of the paper is organized as follows. Section \ref{sec:prelim} reviews the setup of generative domain adaptation and the diffusion model. Section \ref{sec:rgdp} introduces the proposed method and theoretically characterizes its effectiveness. Numerical results are given in Section \ref{sec:exp}. We conclude the paper in Section \ref{sec:conclusion}. All proofs and additional numerical experiments are deferred to the Appendix.





\section{Problem Formulation and Preliminaries}
\label{sec:prelim}

\subsection{Transfer Learning Problem Setup } 
Let $\mathcal{X}$ denote the data space and $\mathcal{Y}$ the label space. 
A domain corresponds to a joint distribution over $\mathcal{X}$ and $\mathcal{Y}$, denoted as $p_{X Y}$ for the {\it source} domain and $q_{X Y}$ for the {\it target} domain. The marginal distribution of data in the source and target domains are $p_{X}$ and $q_{X}$, respectively. Suppose we have access to $m$ (labeled) samples from the source domain
$\mathcal{S}=\left\{\left(\mathbf{x}_i, y_i\right)\right\}_{i=1}^m \sim p_{X Y}$ and $n$ (labeled) samples from the target domain $\mathcal{T}=\left\{\left(\mathbf{x}'_i, y_i'\right)\right\}_{i=1}^n \sim q_{X Y}$. Typically, the source domain contains significantly more samples than the target domain, i.e., $n\ll m$. This setup reflects the common scenario where there is limited data available for specific tasks in the target domain, while abundant data is readily accessible and stored in the source domain.


The problem of interest is as follows. 
Given a pre-trained generative model $p_\theta$ for the data distribution $p_{X}$ in the source domain, and a relatively small number of samples from the target domain, generative domain adaptation approaches aim to obtain a generative model that can generate synthetic samples following the target data distribution $q_{X}$. We will focus on diffusion generative models, given their great success in synthetic data generation. 
We first present the key idea of a carefully designed guidance network for the generation of $\mathbf{x}$ values only. Then, we extend the method to facilitate conditional generations so that we can generate paired samples with labels, $(\mathbf{x},y)$, and can incorporate downstream classification tasks on the target domain. 


\subsection{Preliminaries of Diffusion Model}\label{sub:diffusion}


Diffusion models are characterized by their forward and backward processes. For illustrative purposes, we discuss the diffusion model trained on the source domain. The forward process involves perturbing the data distribution $p_X(\mathbf{x})$ by injecting Gaussian noise, as described by the following continuous-time equation \cite{Song2021ScoreBasedGM}:
\begin{equation}\label{eq:forward}
    \mathrm{d} \mathbf{x}_t=\mathbf{f}(\mathbf{x}_t, t) \mathrm{d} t+g(t) \mathrm{d} \mathbf{w}, \ t\in[0,T],
\end{equation}
where $\mathbf{w}$ is the standard Brownian motion, $\mathbf{f}(\cdot, t): \mathbb{R}^d \rightarrow \mathbb{R}^d$ is a drift coefficient, and $g(\cdot): \mathbb{R} \rightarrow \mathbb{R}$ is a  diffusion coefficient. The marginal distribution of $\mathbf{x}_t$ at time $t$ is denoted as $p_t(\mathbf{x}_t)$, and $p_0$ is the distribution of the initial value $\mathbf{x}_0$, which equals the true data distribution $p_X(\mathbf{x})$. For notational simplicity and provided it does not cause further confusion, we will refer to this diffusion process as $p$ in the following, and we define $p(\mathbf{x}_t|\mathbf{x}_s)$, $\forall s,t$, as the conditional distribution of $\mathbf{x}_t$ given the value $\mathbf{x}_s$. Similarly, for initial value $\mathbf{x}$ following the target domain distribution, we denote the corresponding probability measure induced by the above diffusion process \eqref{eq:forward} as $q$.

Then, we can reverse the forward process \eqref{eq:forward} for generation, defined as:\begin{equation}\label{eq:backward}
\mathrm{d} \mathbf{x}_t=\left[\mathbf{f}(\mathbf{x}_t, t)-g(t)^2 \nabla_{\mathbf{x}} \log p_t(\mathbf{x})\right] \mathrm{d} t+g(t) \mathrm{d} \overline{\mathbf{w}},
\end{equation}
where $\overline{\mathbf{w}}$ is a standard Brownian motion when time flows backwards from $T$ to 0, and $\mathrm{d}t$ is an infinitesimal negative time step. 
The key of the backward process is to estimate the score function of each marginal distribution, $\nabla_{\mathbf{x}} \log p_t(\mathbf{x})$, then the generation can be performed by discretizations of \eqref{eq:backward} \cite{Ho2020DDPM,Song2021ScoreBasedGM}. Score Matching \citep{Hyvrinen2005EstimationON,Vincent2011ACB,Song2019SlicedSM} are proposed to train a neural network $\mathbf{s}_{\boldsymbol{\phi}}(\mathbf{x}_t, t)$ (parameterized by $\boldsymbol{\phi}$) to estimate the score:
\begin{equation}\label{eq:dsm}
\boldsymbol{\phi}^*= \underset{\boldsymbol{\phi}}{\arg \min } \ \mathbb{E}_t\left\{\lambda(t) \mathbb{E}_{p_t(\mathbf{x}_t)} \left[\left\|\mathbf{s}_{\boldsymbol{\phi}}(\mathbf{x}_t, t)-
\nabla_{\mathbf{x}_t} \log p_t(\mathbf{x}_t)\right\|_2^2\right]\right\},
\end{equation}
where $\lambda(t):[0, T] \rightarrow \mathbb{R}_{>0}$ is a positive weighting function, $t$ is uniformly sampled over $[0, T]$. One commonly adopted forward process is choosing an affine $\mathbf{f}(\mathbf{x},t)=-\frac12\beta(t)\mathbf{x}$ and $g(t)=\sqrt{\beta(t)}$, which yields the Gaussian transition distribution $p(\mathbf{x}_t|\mathbf{x}_s)=\mathcal{N}(\mathbf{x}_t;\sqrt{1-\beta(t)} \mathbf{x}_{s}, \beta(t) \mathbb{I})$, $t>s$, with $\beta(t): [0,T] \rightarrow (0,1)$ as a variance schedule. This is the Variance Preserving Stochastic Differential Equation (VP SDE) that we use in the numerical Section \ref{sec:exp}.


 
Several works on image generation \citep{Batzolis2021ConditionalIG, Chao2022DenoisingLS} and inverse problem \citep{ADJ2024CSB} extends Score Matching to Conditional Score Matching, i.e., 
\begin{equation}\label{eq:dsm-cond}
\boldsymbol{\phi}^*= \underset{\boldsymbol{\phi}}{\arg \min } \ \mathbb{E}_t\left\{\lambda(t) \mathbb{E}_{p_t(\mathbf{x}_t, y)} \left[\left\|\mathbf{s}_{\boldsymbol{\phi}}(\mathbf{x}_t, y, t)-
\nabla_{\mathbf{x}_t} \log p_t(\mathbf{x}_t|y)\right\|_2^2\right]\right\},
\end{equation}
where $p_t(\mathbf{x}_t|y)$ is the conditional distribution of perturbed data $\mathbf{x}_t$ given corresponding label $y$.

\section{Transfer Guided Diffusion Process}
\label{sec:rgdp}

In this section, we introduce the proposed Transfer Guided Diffusion Process (TGDP) that leverages a pre-trained diffusion model -- trained on the source domain data -- to generate data in the target domain. The proposed approach is orthogonal to and different from the existing fine-tuning type methods. 
We introduce the additional guidance in Section \ref{sec:guidance}. The methods for calculating the guidance are provided in Section \ref{sec:guidancee_net}. We extend our framework to the conditional diffusion model in Section \ref{sec:conditional} and we propose two regularization terms for enhancing the performance of our method in Section \ref{sec:add_reg}. All proofs are deferred to Appendix \ref{app:proofs}.  


\subsection{Methodology Formulation}
\label{sec:guidance}

This subsection outlines the process of transferring knowledge from a diffusion generative model pre-trained using the source domain data $\mathcal{S}$ for generating samples that match the underlying distribution of target domain sample $\mathcal{T}$. 
The simplest non-transfer type approach involves directly training a diffusion model on samples $\mathcal{T}$ from the target domain by denoising Score Matching as described by Eq \eqref{eq:dsm} or Eq \eqref{eq:dsm-cond}. 
However, since we assume only a limited amount of data is accessible on the target domain, directly learning from the target domain is unlikely to yield an effective generative model. 

Several studies propose to finetune the pre-trained diffusion model to alleviate the challenges caused by limited data and make use of acquired knowledge \citep{Moon2022FinetuningDM,Xiang2023ACL,Zhu2023DomainStudioFD}. These methods typically design different strategies, such as adapters, to avoid finetuning all weights in a pre-trained model.
However, these approaches generally use the pre-trained diffusion model from the source domain only as initial weights. Our method offers a different way for better utilization of the acquired knowledge.

Our proposed method is inspired by the key observation detailed in the following Theorem \ref{thm:IS_guidance}. Intuitively, the score function $\nabla_{\mathbf{x}_t} \log q_t(\mathbf{x}_t)$ for the target domain differs from the score function $\nabla_{\mathbf{x}_t} \log p_t(\mathbf{x}_t)$ of the source domain by a term related to the density ratio function ${q_X}/{p_X}$. We refer to this differing term as a guidance term in the following Theorem. 


\begin{theorem}\label{thm:IS_guidance}
 Consider two diffusion models on the source and target domain, denoted as $p$ and $q$, respectively. Let the forward process on the target domain be identical to that on the source domain,  $q(\mathbf{x}_t | \mathbf{x}_0)=p(\mathbf{x}_t |\mathbf{x}_0)$, and $\mathbf{s}_{\boldsymbol{\phi}^*}(\mathbf{x}_t, t)$ is the score estimator in the target domain:
    \begin{equation}\label{eq:dsm_IS}
\boldsymbol{\phi}^*= \underset{\boldsymbol{\phi}}{\arg \min } ~\mathbb{E}_t\left\{\lambda(t) \mathbb{E}_{q_t(\mathbf{x}_t)} \left[\left\|\mathbf{s}_{\boldsymbol{\phi}}(\mathbf{x}_t, t)-
\nabla_{\mathbf{x}_t} \log q_t(\mathbf{x}_t )\right\|_2^2\right]\right\},
\end{equation}
then we have
\begin{equation}\label{eq:IS_guid}
\mathbf{s}_{\boldsymbol{\phi}^*}(\mathbf{x}_t, t) = \underbrace{\nabla_{\mathbf{x}_t} \log p_t(\mathbf{x}_t)}_{\substack{\text{pre-trained model} \\ \text{on source}}} + \underbrace{\nabla_{\mathbf{x}_t} \log \mathbb{E}_{p(\mathbf{x}_0|\mathbf{x}_t)}\left[\frac{q(\mathbf{x}_0)}{p(\mathbf{x}_0)}\right]}_{\text {guidance}}.
\end{equation}
\end{theorem}
%
%
Based on Eq \eqref{eq:IS_guid}, instead of solving $\mathbf{s}_{\boldsymbol{\phi}^*}$ from the limited training samples on the target domain, we construct  $\mathbf{s}_{\boldsymbol{\phi}^*}$ by combing the pre-training score estimator and the guidance based on a binary classifier of source and target domain samples (detailed in Section \ref{sec:guidancee_net}). 
%
 We comment on some potential advantages of this simple yet effective idea. First of all, we do not need to fine-tune the pre-trained diffusion model on the source domain, with the corresponding computation shifted to training the guidance network which is essentially a classifier. Second, the guidance network can be effectively estimated by a domain classifier using data from both the source and target domains. There is also great flexibility in constructing this guidance network due to the extensive literature on classification problems and density ratio estimation approaches. Additionally, the sample complexity for training a generative model could be much larger than a discriminative model, since the generative model needs to recover the full spectrum of target data distribution, while a domain classifier only needs to distinguish whether the sample is from the source or target distribution. 

\subsection{Learning Guidance Network}
\label{sec:guidancee_net}

We calculate the guidance for the diffusion model on the target domain as defined in the second term of Eq \eqref{eq:IS_guid} via two steps. In the first step, we estimate the density ratio ${q(\mathbf{x}_0)}/{p(\mathbf{x}_0)}$ by training a classifier $c_{\boldsymbol{\omega}}(\mathbf{x}):\mathcal X \to [0,1]$ to distinguish samples from the source and target domains. We adopt the typical logistic loss as follows:
\begin{equation} \label{eq:logistic}
 \boldsymbol{\omega}^*= \underset{\boldsymbol{\omega}}{\arg \min }  \left\{-\frac{1}{m}\sum_{\mathbf{x}_i\sim p}\log c_{\boldsymbol{\omega}}(\mathbf{x}_i)-\frac{1}{n}\sum_{\mathbf{x}'_i\sim q}\log (1-c_{\boldsymbol{\omega}}(\mathbf{x}'_i)) \right\}.
\end{equation}
Then, the density ratio ${q(\mathbf{x}_0)}/{p(\mathbf{x}_0)}$ can be estimated as ${(1-c_{\boldsymbol{\omega}^*}(\mathbf{x}_0))}/{c_{\boldsymbol{\omega}^*}(\mathbf{x}_0)}$, and it can be shown that the optimal solution to the population counterpart of Eq \eqref{eq:logistic} is exactly the true likelihood ratio \citep{DensityRatio2012}. It is worthwhile mentioning that we may only use a subset of source domain samples to learn the classifier $c_{\boldsymbol{\omega}}$ to alleviate the unbalanced sample sizes, and we could also adopt modern density ratio estimators to improve the accuracy \cite{NEURIPS2020TRE}.
%
After learning the density ratio ${q(\mathbf{x}_0)}/{p(\mathbf{x}_0)}$, the second step is to calculate the expectation $\mathbb{E}_{p(\mathbf{x}_0|\mathbf{x}_t)}[q(\mathbf{x}_0)/p(\mathbf{x}_0)]$ using Monte Carlo simulation. Since it is hard to sample from $q(\mathbf{x}_0 | \mathbf{x}_t)$, we use the following equivalent formulation to get the value instead. This trick has also been used in previous work such as the Appendix H in \cite{Lu2023ContrastiveEP}. 
\begin{lemma}\label{thm:exact_guidance}
For a neural network $h_{\boldsymbol{\psi}}\left(\mathbf{x}_t, t\right)$ parameterized by $\boldsymbol{\psi}$, define the objective  
\begin{equation} \label{eq:guidance}
\mathcal{L}_{\text{guidance}}(\boldsymbol{\psi}) :=\mathbb{E}_{p(\mathbf{x}_0, \mathbf{x}_t)}\left[\left\|h_{\boldsymbol{\psi}}\left(\mathbf{x}_t, t\right)-\frac{q(\mathbf{x}_0)}{p(\mathbf{x}_0)}\right\|_2^2\right],
\end{equation}
then its minimizer $\boldsymbol{\psi}^* = \underset{\boldsymbol{\psi}}{\arg \min } \ \mathcal{L}_{\text{guidance}}(\boldsymbol{\psi})$ satisfies:
\[
h_{\boldsymbol{\psi}^*}\left(\mathbf{x}_t, t\right)=\mathbb{E}_{p(\mathbf{x}_0 |\mathbf{x}_t)}\left[{q(\mathbf{x}_0)}/{p(\mathbf{x}_0)}\right].
\]
\end{lemma}
By Lemma \ref{thm:exact_guidance}, we estimate the value $\mathbb{E}_{p(\mathbf{x}_0 |\mathbf{x}_t)}\left[{q(\mathbf{x}_0)}/{p(\mathbf{x}_0)}\right]$ using the guidance network $h_{\boldsymbol{\psi}^*}$ solved by minimizing the objective function $\mathcal{L}_{\text{guidance}}(\boldsymbol{\psi})$, which can be approximated by easy sampling from the joint distribution $p(\mathbf{x}_0, \mathbf{x}_t)$. Combine the above steps together, the estimated score function for the diffusion generative model on target domain $q_{X}$ can be calculated as follows:
\begin{equation}\label{eq:dsm_IS_sampling}
\mathbf{s}_{\boldsymbol{\phi}^*}(\mathbf{x}_t, t) = \underbrace{\nabla_{\mathbf{x}_t} \log p(\mathbf{x}_t )}_{\substack{\text {pre-trained model}\\ \text{on source}}}+ \underbrace{\nabla_{\mathbf{x}_t} \log h_{\boldsymbol{\psi}^*}\left(\mathbf{x}_t, t\right)}_{\text {guidance network}}.
\end{equation}


\subsection{Extension to the Conditional Version}
\label{sec:conditional}

The approach outlined above is for generating the sample $\mathbf{x}$ in the target domain. In this section, we extend the idea to the conditional generation task. Such extension is essential when the label sets in the source and target domain are different since, in such cases, we usually rely on the conditional diffusion model for sampling \citep{Khachatryan2023Text2VideoZeroTD, Li2023YourDM}. 
%
We first present the following theorem, which is an analog to Theorem \ref{thm:IS_guidance} within the context of conditional score matching. 
\begin{theorem}\label{thm:IS_guidance_conditional}
    Assume $\mathbf{x}_t$ and $y$ are conditional independent given $\mathbf{x}_0$ in the forward process, i.e.,  $p(\mathbf{x}_t|\mathbf{x}_0,y)=p(\mathbf{x}_t|\mathbf{x}_0)$, $\forall t\in[0,T]$, and let the forward process on the target domain be identical to that on the source domain $q(\mathbf{x}_t | \mathbf{x}_0)=p(\mathbf{x}_t | \mathbf{x}_0)$, and $\boldsymbol{\phi}^*$ is the optimal solution for the conditional diffusion model trained on target domain $q(\mathbf{x}_0, y)$, i.e.,
    \begin{equation}
\boldsymbol{\phi}^*= \underset{\boldsymbol{\phi}}{\arg \min } ~\mathbb{E}_t\left\{\lambda(t) \mathbb{E}_{q_t(\mathbf{x}_t,y)} \left[\left\|\mathbf{s}_{\boldsymbol{\phi}}(\mathbf{x}_t, y, t)-
\nabla_{\mathbf{x}_t} \log q_t(\mathbf{x}_t | y)\right\|_2^2\right]\right\},
\end{equation}
then 
\begin{equation}\label{eq:guide_conditional}
\begin{aligned}
\mathbf{s}_{\boldsymbol{\phi}^*}(\mathbf{x}_t, y, t) = \underbrace{\nabla_{\mathbf{x}_t} \log p_t(\mathbf{x}_t |  y)}_{\substack{\text {pre-trained conditional model}\\ \text{on source}}}+ \underbrace{\nabla_{\mathbf{x}_t} \log \mathbb{E}_{p(\mathbf{x}_0|\mathbf{x}_t, y)}\left[\frac{q(\mathbf{x}_0, y)}{p(\mathbf{x}_0, y)}\right]}_{\text {conditional guidance}}.
\end{aligned}    
\end{equation}
\end{theorem}
The key difference is we need to estimate the joint density ratio between the source and target domain. We can extend the density ratio estimator in Section \ref{sec:guidancee_net} for estimating joint density ratio, i.e., also feed the label $y$ into the classifier $c_{\boldsymbol{\omega}}(\mathbf{x},y)$.
The corresponding Lemma and its proof for the conditional version of Lemma \ref{thm:exact_guidance} can be found in Appendix \ref{ap:conditional_exact_guidance}. We further provide a detailed discussion about how to extend this conditional guidance to text-to-image generation tasks and when the source and target domain contain different class labels in Appendix \ref{ap:potential}.

\subsection{Additional Regularizations in Practical Implementations}\label{sec:add_reg}

In this subsection, we provide two additional regularization terms in our final objective function, to enhance the performance of the proposed scheme. 

\paragraph{Cycle Regularization}
In the approaches described above, after obtaining the classifier network $c_{\boldsymbol{\omega}^*}$, calculation of the additional guidance $\nabla_{\mathbf{x}_t} \log \mathbb{E}_{p(\mathbf{x}_0|\mathbf{x}_t)}[{q(\mathbf{x}_0)}/{p(\mathbf{x}_0)}]$(or $\nabla_{\mathbf{x}_t} \log \mathbb{E}_{p(\mathbf{x}_0|\mathbf{x}_t, y)}[{q(\mathbf{x}_0, y)}/{p(\mathbf{x}_0, y)}]$ for conditional generation) {\it only} utilizes the data from source domain. In this section, we provide an enhancement in which the limited data from the target domain can also be utilized to improve the training of the guidance network $h_{\boldsymbol{\psi}}$. 

Notice that (with detailed derivation given in Appendix \ref{ap:proof_cycle_regularization})
\begin{equation}\label{eq:cycle_reg}
\mathbb{E}_{p(\mathbf{x}_0|\mathbf{x}_t)}\left[\frac{q(\mathbf{x}_0)}{p(\mathbf{x}_0)}\right] = \mathbb{E}_{q(\mathbf{x}_0|\mathbf{x}_t)}\left[\frac{q_t(\mathbf{x}_t)}{p_t(\mathbf{x}_t)}\right],
\end{equation}
where recall $p_t(\mathbf{x}_t)$ and $q_t(\mathbf{x}_t)$ are the marginal distribution at time $t$ for source and target distributions, respectively. A similar idea to Theorem \ref{thm:exact_guidance} implies that we can learn the guidance network by solving the following optimization problem as well:
\begin{equation}\label{eq:cycle}
\boldsymbol{\psi}^* = \underset{\boldsymbol{\psi}}{\arg \min }\ \mathcal{L}_{\text{cycle}}:=\mathbb{E}_{q\left(\mathbf{x}_0, \mathbf{x}_t\right)}\left[\left\|h_{\boldsymbol{\psi}}\left(\mathbf{x}_t, t\right)-\frac{q_t(\mathbf{x}_t)}{p_t(\mathbf{x}_t)}\right\|_2^2\right].
\end{equation}



Moreover, in order to estimate the density ratio for marginal distributions at time $t$ between the target and source data distribution, we train a time-dependent classifier $c_{\boldsymbol{\omega}}(\mathbf{x},t)$ to distinguish samples from source domain $p_t(\mathbf{x})$ and target domain $q_t(\mathbf{x})$ by the logistic loss as follow: 
\[
 \boldsymbol{\omega}^*= \underset{\boldsymbol{\omega}}{\arg \min }  \left\{-\frac{1}{m}\sum_{\mathbf{x}_0\sim p}\sum_{\mathbf{x}_t |\mathbf{x}_0}\log c_{\boldsymbol{\omega}}(\mathbf{x}_t,t) -\frac{1}{n}\sum_{\mathbf{x}_0\sim q}\sum_{\mathbf{x}_t |\mathbf{x}_0}\log (1-c_{\boldsymbol{\omega}}(\mathbf{x}_t,t)) \right\},
\]
where $m,n$ are the number of training samples in source and target domains. The density ratio $q_t(\mathbf{x}_t)/p_t(\mathbf{x}_t)$ can then be calculated by $(1-c_{\boldsymbol{\omega}^*}(\mathbf{x}_t,t))/(c_{\boldsymbol{\omega}^*}(\mathbf{x}_t,t))$.

\paragraph{Consistency Regularization}
\label{sec:consis_reg}


Motivated by the fact that an optimal guidance network should recover the score in the target domain, we further use score matching in the target domain as the Consistency Regularization $\mathcal{L}_{\text{consistence}}$ to learn the guidance network better. 
\begin{align}\label{eq:consistence}
\boldsymbol{\psi}^* &= \underset{\boldsymbol{\psi}}{\arg \min }\ \mathcal{L}_{\text{consistence}}\notag\\
&:= \mathbb{E}_t\left\{\lambda(t) \mathbb{E}_{q(\mathbf{x}_0)} \mathbb{E}_{q(\mathbf{x}_t|\mathbf{x}_0)}\left[\left\|\nabla_{\mathbf{x}_t} \log p(\mathbf{x}_t | \mathbf{x}_0)+ \nabla_{\mathbf{x}_t} \log h_{\boldsymbol{\psi}}\left(\mathbf{x}_t, t\right)-
\nabla_{\mathbf{x}_t} \log q(\mathbf{x}_t | \mathbf{x}_0)\right\|_2^2\right]\right\}.
\end{align}

Combining these two additional regularization terms together with the original guidance loss \eqref{eq:guidance}, the final learning objective for the guidance network can be described as follows:
\begin{equation}\label{eq:final_obj}
\boldsymbol{\psi}^* = \underset{\boldsymbol{\psi}}{\arg \min } \ \{\mathcal{L}_{\text{guidance}}+ \eta_1 \ \mathcal{L}_{\text{cycle}} + \eta_2 \ \mathcal{L}_{\text{consistence}} \},  
\end{equation}
where $\eta_1, \eta_2 \geq 0$ are hyperparameters that control the strength of additional regularization, which also enhances the flexibility of our solution scheme. We summarize the Algorithm of TGDP in Appendix \ref{app:alg_box}. We provide the ablation studies that demonstrate the effectiveness of these two regularizations in Appendix \ref{ap:regu} and we also empirically show only adopt $\ \mathcal{L}_{\text{consistence}}$ to optimize the guidance network is not good enough because of the limited data from the target distribution.





\begin{remark}[Discussion about related guidance]
Classifier guidance has become a common trick in recent research \citep{Song2021ScoreBasedGM, Dhariwal2021DiffusionMB, Bansal2023UniversalGF, Chung2022DiffusionPS}. We highlight that, under the transfer learning framework, the guidance proposed in our work is the optimal guidance since the resulting score function matches the oracle score on the target domain. On the contrary, vanilla versions of classifier guidance utilizing a domain classifier cannot generate samples that exactly follow target distribution. Indeed, 
for a pre-trained domain classifier $c_{\boldsymbol{\omega}}$, vanilla domain classifier guidance formulates the source for generation as follows:
\[
\begin{aligned}
\mathbf{s}_{\boldsymbol{\phi}^*}(\mathbf{x}_t, t) &=\nabla_{\mathbf{x}_t} \log p(\mathbf{x}_t )- \nabla_{\mathbf{x}_t}  \mathbb{E}_{p(\mathbf{x}_0|\mathbf{x}_t)}\left[\log (1-c_{\boldsymbol{\omega}}(\mathbf{x}_0))\right]\\
&\neq \nabla_{\mathbf{x}_t} \log p(\mathbf{x}_t )+ \nabla_{\mathbf{x}_t} \log \mathbb{E}_{p(\mathbf{x}_0 |\mathbf{x}_t)}\left[\frac{q(\mathbf{x}_0)}{p(\mathbf{x}_0)}\right] \quad \text{(correct form proven in Theorem \ref{thm:IS_guidance})}\\
&= \nabla_{\mathbf{x}_t} \log p(\mathbf{x}_t )+ \nabla_{\mathbf{x}_t} \log \mathbb{E}_{p(\mathbf{x}_0|\mathbf{x}_t)}\left[\frac{1-c_{\boldsymbol{\omega}}(\mathbf{x}_0))}{c_{\boldsymbol{\omega}}(\mathbf{x}_0)}\right].
\end{aligned}
\]
\end{remark}



\section{Experiments}
\label{sec:exp}

In this section, we present empirical evidence demonstrating the efficacy of the proposed Transfer Guided Diffusion Process (TGDP) on limited data from a target domain. In Section \ref{exp:simulation}, we conduct proof-of-concept experiments using a Gaussian mixture model to showcase that the guidance network of TGDP can successfully steer the pre-trained diffusion model toward the target domain. In Section \ref{exp:ecg}, we illustrate the effectiveness of TGDP using a real-world electrocardiogram (ECG) dataset. 

\subsection{Simulation Results}
\label{exp:simulation}

\paragraph{Experimental setup}

We begin with a Gaussian mixture model where $\mathcal{X}=\mathbb{R}^d$ and $\mathcal{Y}=\{-1,1\}$. On both domains, the marginal distribution for label $y$ is uniform in $\mathcal{Y}$, and the conditional distribution of features is $\mathbf{x}|y\sim \mathcal{N}(y\boldsymbol{\mu},\sigma^2 \mathbb{I}_d)$, where $\boldsymbol{\mu}\in \mathbb{R}^d$ is non-zero, and $\mathbb{I}_d$ is the $d$ dimensional identity covariance matrix. We let $\boldsymbol{\mu}=\boldsymbol{\mu}_{\rm S}$ on the source domain and $\boldsymbol{\mu}=\boldsymbol{\mu}_{\rm T}$ on the target domain, with $(\boldsymbol{\mu}_{\rm S})\top\boldsymbol{\mu}_{\rm T}=0$.
Under such case, the marginal distribution of $\mathbf{x}$ on the source domain $p_{X}$ is a Gaussian mixture, for convenience we denote it as $0.5\mathcal{N}({\boldsymbol{\mu}_{\rm S}},\sigma^2 \mathbb{I})+0.5 \mathcal{N}(-{\boldsymbol{\mu}_{\rm S}},\sigma^2 \mathbb{I})$, and the marginal feature of target distribution $q_{X}$ is $0.5\mathcal{N}({\boldsymbol{\mu}_{\rm T}},\sigma^2 \mathbb{I})+0.5 \mathcal{N}(-{\boldsymbol{\mu}_{\rm T}},\sigma^2 \mathbb{I})$. We let $d=2, \boldsymbol{\mu}_{\rm S}=[0.5,0.5], \boldsymbol{\mu}_{\rm T}=[0.5,-0.5],\sigma^2=0.1$, and draw $m=10000$ samples from source domain $p_{X}$, and $n=10$, $100$, $1000$ samples from target domain $q_{X}$, respectively. 


\paragraph{Implementation details and Baselines}
We adopt the default Variance Preserving (VP) SDE in \cite{Song2021ScoreBasedGM} with a linear schedule, i.e., $q(\mathbf{x}_t | \mathbf{x}_0)=p(\mathbf{x}_t |\mathbf{x}_0)=\mathcal{N}\left(\mathbf{x}_t| \alpha_t \mathbf{x}_0, \sigma_t^2 \boldsymbol{I}\right)$ with $\alpha_t$ and $\sigma_t$ being:
$$
\alpha_t=-\frac{\beta_1-\beta_0}{4} t^2-\frac{\beta_0}{2} t, \quad \sigma_t=\sqrt{1-\alpha_t^2}, 
$$with  $\beta_0=0.1, \beta_1=20$. We adopt 5-layer MLP with hidden sizes of $[512,512,512,512,256]$ and SiLU activation function as the diffusion model. We train the diffusion model on data from the source domain for 100 epochs using the Adam optimizer with a learning rate of $1 \mathrm{e}^{-4}$ and batch size of 4096. The guidance network is a 4-layer MLP with 512 hidden units and SiLU activation function. We train the guidance network 20 epochs for our TGDP and train a vanilla diffusion model or finetune the diffusion model target domain 50 epochs. For generation, we adopt DPM-Solver \citep{Lu2022DPMSolverAF} with a second-order sampler and a diffusion step of 25.
%
We compare TGDP with the following baseline methods: 1) Vanilla Diffusion: directly training from target domain; 2) Finetune Diffusion: finetuning all weights of a pre-trained diffusion model on target distribution \footnote{It is worthwhile mentioning that the reason we do not compared with the works that finetune partial weights in a pre-trained diffusion model \citep{Xie2023DiffFitUT} is their results are usually worse or comparable with method that finetunes all weights, the implementation of \citep{Moon2022FinetuningDM} are not available, and the regularization proposed by \citep{Zhu2022FewshotIG} is only valid for image data.}.

\paragraph{Experimental results}   
We first demonstrate the effectiveness of guidance in Figure \ref{fig:simulation} under the above setup. Figure \ref{fig:simulation} (a) plots the source samples, while Figure \ref{fig:simulation} (b) shows the target samples under different sample sizes $n=10$, $100$, $1000$. Figure \ref{fig:simulation} (c-e) illustrates the generated target samples via different methods, respectively. It can be seen that the samples generated via the proposed TGDP approach share similar patterns with the target distribution and two mixture components are more obvious as compared with other baseline methods. Furthermore, since the true data distribution of the target domain is known, we calculate the average likelihood of samples generated by each method as demonstrated in Table \ref{tab:likelihood} for quantitative evaluation and comparison.


\begin{figure*}[h!]
\centering
  \includegraphics[width=0.9\textwidth]{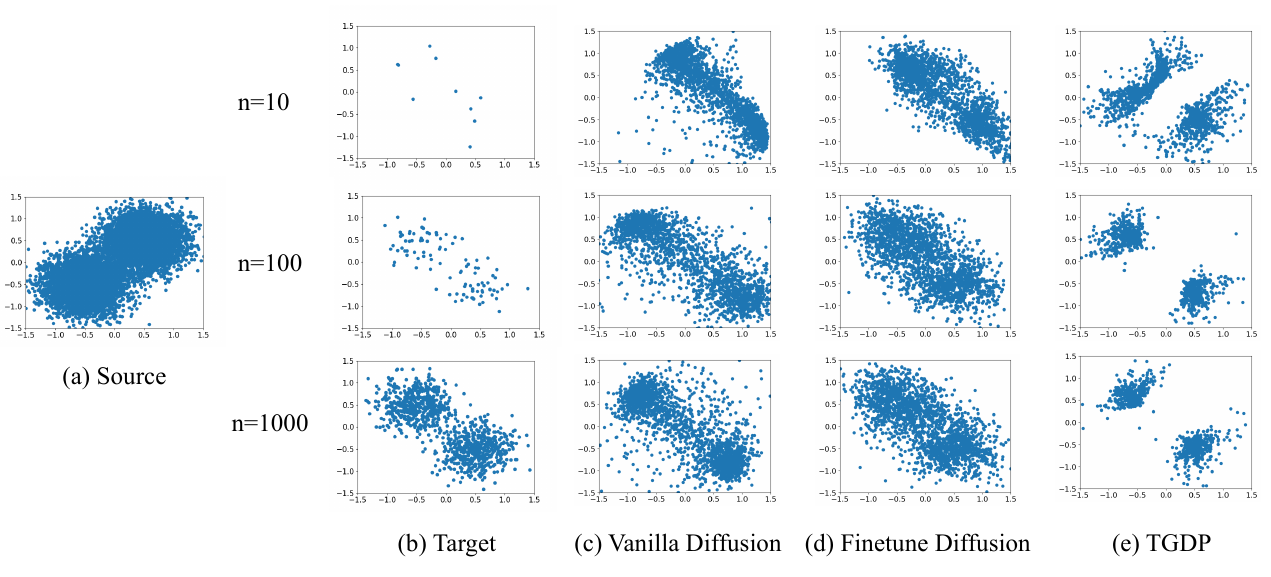}
  \caption{An illustration of the effectiveness of TGDP on simulations with 10/100/1000 target samples, respectively.}
  \label{fig:simulation}
\end{figure*}

\begin{table}[htbp]
\centering
\caption{Quantitative evaluation of TGDP on simulations. Training on 10K samples from the source domain and $n=10,100,1000$ numbers on the target domain, respectively. TGDP achieves the highest average likelihood under target distribution.}
 \begin{tabular}{c|ccc}
\toprule  &  \multicolumn{3}{c}{Average likelihood}  \\
& \multicolumn{1}{l}{n=10} & \multicolumn{1}{l}{n=100}  & \multicolumn{1}{l}{n=1000}  \\
\midrule
Vanilla Diffusion   & 0.145 & 0.253 & 0.328 \\ 
Finetune &0.290 &0.329 & 0.335\\  
TGDP &  \textbf{0.417} &\textbf{0.627} & \textbf{0.673}\\ 
\bottomrule
\end{tabular}
\label{tab:likelihood}
\end{table}


As a sanity check, we also look at the sensitivity of the learned density ratio estimator (through the classifier network \eqref{eq:logistic}) regarding different sizes of target samples. As shown in Figure \ref{ab:densityratio}, even with only 10 samples from the target domain (and 10 samples from the source domain for class balance sampling), we can accurately estimate the landscape of density ratio (although the magnitude of the estimated ratio is not entirely accurate when the number of target samples equal 10). 

\begin{figure}[th]
\centering
\subfigure[Oracle]{\includegraphics[width=0.22\textwidth]{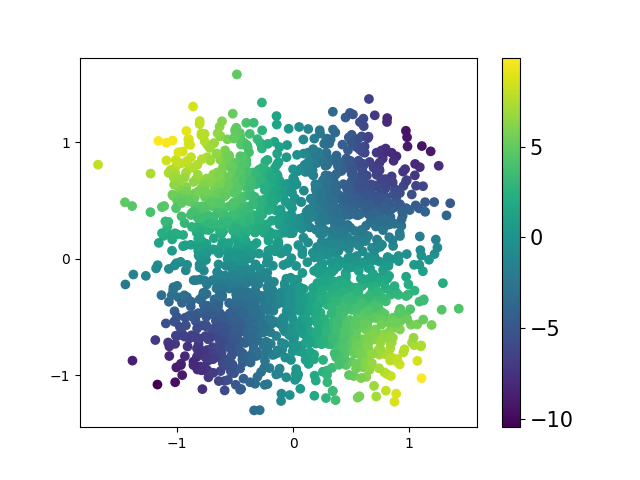}}
\subfigure[10 target]{\includegraphics[width=0.22\textwidth]{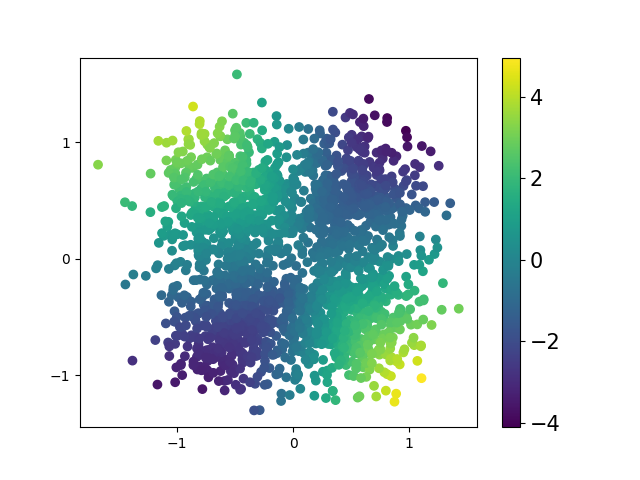}}
\subfigure[100 target]{\includegraphics[width=0.22\textwidth]{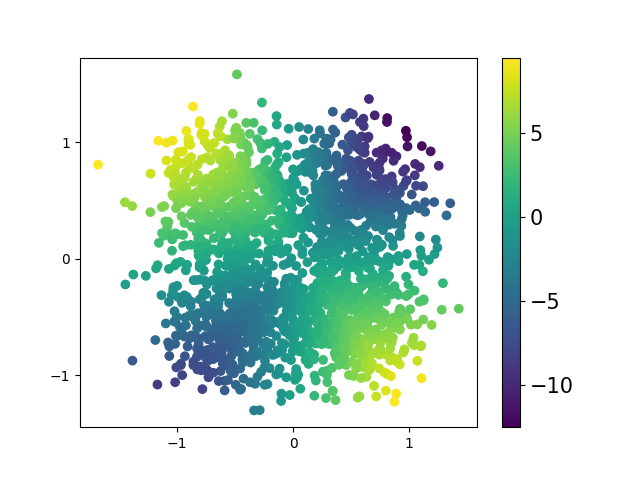}}
\subfigure[1000 target]{\includegraphics[width=0.22\textwidth]{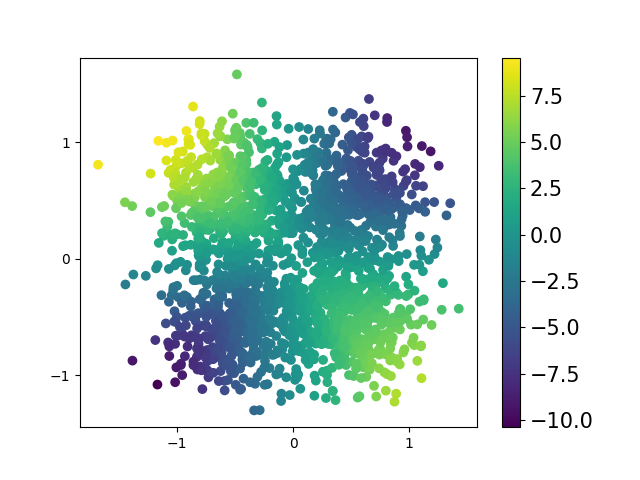}}
\caption{An ablation study of the sensitivity of density ratio estimator.}
\label{ab:densityratio}
\end{figure}

\subsection{ECG Data}
\label{exp:ecg}
In this section, we demonstrate the effectiveness of the proposed guidance on the benchmark of electrocardiogram (ECG) data. We first provide the standard synthetic quality and diversity evaluation in Section \ref{sec:quality}. Then, we utilize downstream classification tasks to further evaluate the effectiveness of TGDP in Section \ref{sec:downstream}.
%
We follow the setup of existing benchmarks on biomedical signal processing \citep{Strodthoff2020DeepLF} that regard PTB-XL dataset \cite{Wagner2020PTBXLAL} as the source domain and ICBEB2018 dataset \citep{Ng2018AnOA} as the target domain. PTB-XL dataset contains 21,837 clinical 12-lead ECG recordings of 10 seconds length from 18,885 unique patients. A 12-lead ECG refers to the 12 different perspectives of the heart's electrical activity that are recorded. Moreover, the PTB-XL dataset is a multi-label dataset with 71 different statements (label). ICBEB2018 dataset \citep{Ng2018AnOA} comprises 6877 12-lead ECGs lasting between 6 and 60 seconds. Each ECG record is categorized into one of nine classes, which is a subset of labels in the PTB-XL dataset. We randomly select 10\% samples as limited target distribution by stratified sampling preserving the overall label distribution in each fold following \citep{Wagner2020PTBXLAL}. We use the data from PTB-XL dataset and ICBEB2018 dataset at a sampling frequency of 100 Hz, which means 100 samples per second. We include more implementation details in Appendix \ref{ap:imple_detail}.

\subsubsection{Synthetic Quality and Diversity Evaluation}
\label{sec:quality}

\paragraph{Baseline method} 
We compare TGDP with the following baseline methods to demonstrate the effectiveness of TGDP. 1) Learn a generative model directly {\it(Vanilla Diffusion)}: The vanilla way is to learn a generative model directly on limited samples from the target domain. 2) Leveraging the pre-trained generative model from source domain {\it(Finetune Generator}): Since the label set of the target domain is a subset of that in the source domain, a preliminary solution is to utilize the pre-trained diffusion model to generate samples with labels in the target domain.




\paragraph{Experimental results}

In Table \ref{tab:acc_fid_div}, we compare the generation performance on the target domain using two metrics. The first criterion is the widely used Frechet Inception Distance (FID) \citep{Heusel2017GANsTB} to evaluate the quality of synthetic data, which calculates the Wasserstein-2 distance between the real data and the synthetic data on the feature space. We use the pre-trained classifier on the target domain as the feature extractor, i.e., xresnet1d50 \cite{Strodthoff2020DeepLF}. The second metric is the coverage \citep{Naeem2020ReliableFA} that evaluates the diversity of the synthetic data. It is defined as the ratio of real records that have at least one fake (synthetic) record in its sphere. The higher the coverage is, the more diverse the synthetic data are.

From Table \ref{tab:acc_fid_div}, we see that TGDP achieves better performance than baseline methods on two criteria, which demonstrates the effectiveness of TGDP on generative transfer learning in scenarios with limited data. Moreover, TGDP has fewer parameters to be trained and less training time. We also demonstrate the T-SNE of the generated ECG data in Figure \ref{fig:tsne}. 

\begin{figure}[th]
\centering
\subfigure[Vanilla Diffusion]{\includegraphics[width=0.23\textwidth]{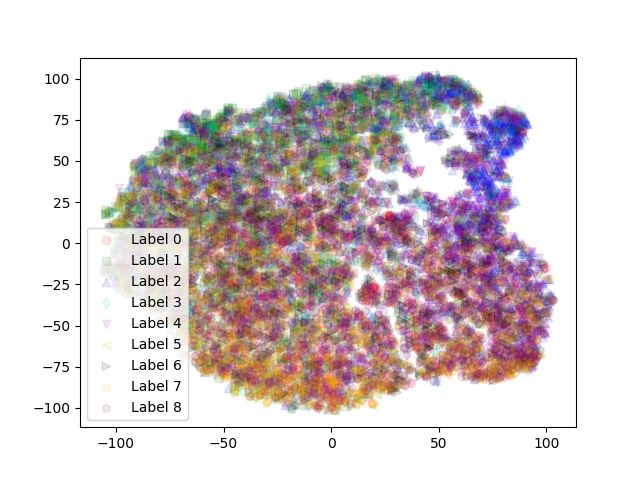}}
\subfigure[Finetune Generator]{\includegraphics[width=0.23\textwidth]{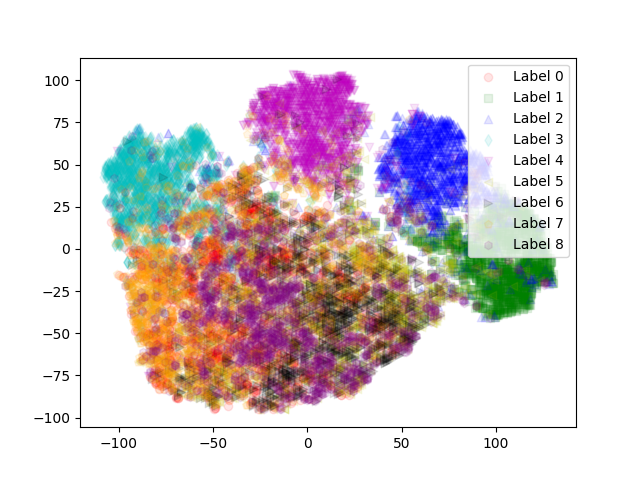}}
\subfigure[TGDP]{\includegraphics[width=0.23\textwidth]{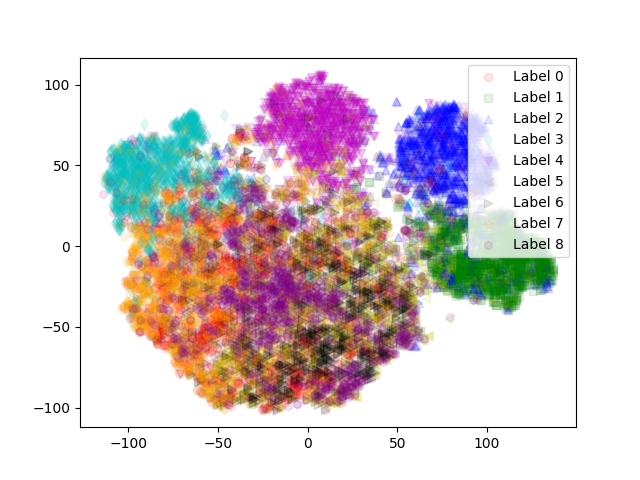}}
\subfigure[Real]{\includegraphics[width=0.23\textwidth]{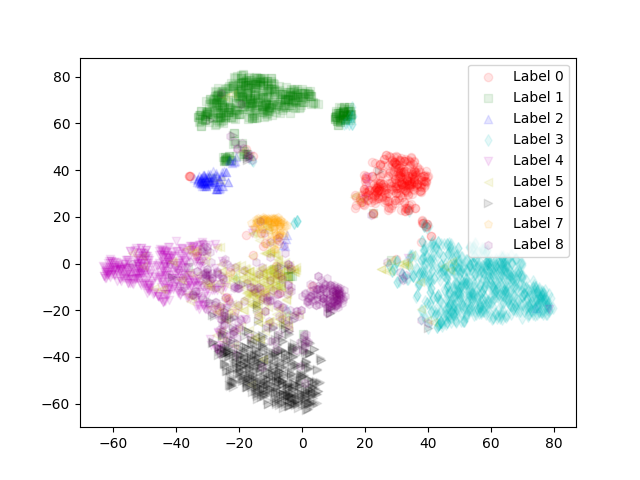}}
\caption{T-SNE of the generated ECG data.}
\label{fig:tsne}
\end{figure}

\begin{table}[htbp]
\centering
 \caption{The effectiveness of TGDP on ECG benchmark under synthetic quality and diversity criteria.} 
\begin{tabular}{c|c|c|c|c}
\toprule  Method   & Diversity ($\uparrow$) & FID ($\downarrow$)& Number of Parameters & Training Time \\
\midrule
Vanilla  Diffusion   &    0.37 & 11.01 &  50.2M&  1h\\ 
Finetune Generator &  0.47  &12.26 & 50.2M & 40min \\  %
TGDP   &  \textbf{0.53} & \textbf{10.46} &  \textbf{2.8M}  & 30min \\ 
\bottomrule
\end{tabular}
\label{tab:acc_fid_div}
\end{table}





\subsubsection{TGDP for Downstream Task} \label{sec:downstream}

In Section \ref{sec:quality}, we illustrate that TGDP is capable of generating samples that adhere to the joint distribution of data and labels in the target domain and is diverse enough. In this subsection, we further investigate whether utilizing TGDP to acquire a generative model for the target domain yields superior performance compared to existing transfer learning pipelines. 

\paragraph{Baseline method} First of all, we can utilize the generative model learned in Section \ref{sec:quality} to generate sufficient samples. Incorporated with the original limited sample from the target domain, we can train the classifier, which we still denoted as {\it Vanilla Diffusion}, {\it Finetune Generator}, and TGDP, respectively. Moreover, we have the following baseline methods. Directly train a classifier on target domain {\it (Vanilla Classifier)}: Utilizing the limited data from the target domain, a vanilla classifier can be obtained. Finetune pre-trained classifier {\it (Finetune Classifier)}: Instead of training a classifier from scratch on the target domain, the parameters of the classifier trained on the source domain are adjusted by using the limited data from the target domain. To verify the effectiveness of the generative model, we demonstrate that it improves the performance of the learned classifier in the following.





\paragraph{Experimental results} 

We adopt the same evaluation criteria as ECG benchmark \cite{Strodthoff2020DeepLF}, i.e., Macro-averaged area under the receiver operating characteristic curve (AUC), Macro-averaged $F_\beta$-score $(\beta=2)$, where $F_\beta=\frac{\left(1+\beta^2\right) \cdot \text { TP}}{\left(1+\beta^2\right) \cdot \text { TP }+\beta^2 \cdot \text { FN }+ \text { FP }}$, and Macro-averaged $G_\beta$-score with $\beta=2$, where $G_\beta=$ $ \frac{\text { TP} }{ \text { TP}+\text { FP }+\beta \cdot \text { FN }}$.
In Table \ref{tab:utility_newest}, TGDP outperforms baseline methods across three evaluation criteria, showcasing its effectiveness in transfer for diffusion model with limited data.



\begin{table}[htbp]
\centering
 \caption{The effectiveness of TGDP on ECG benchmark for downstream task. We provide 95\% confidence intervals via empirical bootstrapping used by \cite{Strodthoff2020DeepLF}. 0.906(03) stands for 0.906 ± 0.003.}

\begin{tabular}{c|c|c|c}
\toprule  Method & AUC & $F_{\beta=2}$  &  $G_{\beta=2}$  \\
\midrule
Vanilla Classifier         & 0.906(03)   &    0.674(06)     & 0.433(06)    \\
Finetune Classifier     &  0.941(05) & 0.747(08) & 0.521(10)  \\
\midrule
Vanilla Diffusion     &  0.932(05) &  0.718(09) & 0.464(09) \\ 
Finetune Generator & 0.941(04)  & 0.761(10) & 0.528(12) \\ 
TGDP &  \textbf{0.953}(05) & \textbf{0.773}(11) &  \textbf{0.534}(11) \\ 
\bottomrule
\end{tabular}
\label{tab:utility_newest}
\end{table}




\section{Conclusion}
\label{sec:conclusion}

In this work, we propose a novel framework, Transfer Guided Diffusion Process (TGDP), for transferring a source-domain diffusion model to the target domain which consists of limited data. Instead of reducing the finetuning parameters or adding regularization for finetuning, TGDP proves the optimal diffusion model on the target domain is the pre-trained diffusion model on the source domain with additional guidance. 
TGDP outperforms existing methods on Gaussian mixture simulations and electrocardiogram (ECG) data benchmarks.
\paragraph{Limitations and broader impact}
Overall, this research presents a promising direction for leveraging pre-trained diffusion models to tackle new tasks. The proposed method, TGDP, has potential applications in a wide range of tasks where domain shift exists. A limitation of this study is the lack of empirical validation regarding TGDP's performance on language vision tasks, which we have earmarked for future exploration. Since we propose a generic algorithm for transferring knowledge to new tasks, this technique could enable people to train Deepfakes for disinformation better. Our approach hinges on the efficacy of detection methods in mitigating negative societal consequences.

\section{Acknowledgement}
Hongyuan Zha was supported in part by the Shenzhen Key Lab of Crowd Intelligence Empowered Low-Carbon Energy Network (No.
ZDSYS20220606100601002). Guang Cheng was partially sponsored by NSF – SCALE MoDL (2134209), NSF – CNS (2247795), Office of Naval Research (ONR N00014-22-1-2680) and CISCO and Optum AI Research Grants.


\newpage
\appendix
\onecolumn

\section{More Discussion on Related Work}\label{app:related}

\paragraph{Finetune diffusion model on limited data}
Directly finetuning the pre-trained generative model on limited data from the target domain may suffer from overfitting and diversity degradation. Moon et. al. \citep{Moon2022FinetuningDM} introduce a time-aware adapter inside the attention block. Since the attention modules take about 10\% of parameters in the
entire diffusion model, they significantly reduced the turning parameters and alleviated the overfitting. While in \cite{Xie2023DiffFitUT}, the authors only finetune specific parameters related to bias, class embedding, normalization, and
scale factor. Zhu et. al. \citep{Zhu2022FewshotIG} found out the images generated by directly finetuned diffusion models share similar features like facial expressions and lack ample high-frequency details. Therefore, they introduce two regularization terms, pairwise similarity
loss for diversity and high-frequency components loss to enhance the high-frequency feature.

The main drawback of finetuning the pre-trained diffusion model is the sample complexity is larger compared with training a classifier since modeling the distribution is a tougher task. In our work, we decompose the diffusion model on the target domain as the diffusion model on the source domain plus a guidance network. Since training a guidance network (essential as a classifier demonstrated in section \ref{sec:rgdp}) requires smaller sample complexity, we believe this novel framework might provide a new way for diffusion-based domain adaptation.

\paragraph{Text-to-image diffusion model and learning with human feedback}
Numerous studies on Text-to-Image diffusion models focus on optimizing the diffusion model to align with human preferences and personalize its performance for specific tasks. These endeavors commonly involve strategies such as text-guided zero-shot finetuning \citep{Song2022DiffusionGD, Radford2021LearningTV} or finetuning diffusion model (or its adaptor) through reward-weighted objectives \citep{Ruiz2022DreamBoothFT, Kumari2022MultiConceptCO, Gal2022AnII, Lee2023AligningTM, Fan2023DPOKRL}. We acknowledge the significant potential in these approaches, given that language models inherently encapsulate rich semantic information, thereby endowing text-to-image diffusion models with zero-shot transferability. However, it is noteworthy that in domains lacking a substantial amount of paired data for learning semantic mappings, such as biomedical signal processing and electrocardiogram (ECG) data, we refrain from considering these methods as the primary benchmarks in our comparative analysis.

\paragraph{Non-diffusion based approaches in generative domain adaptation}
Numerous works in generative domain adaptation (or few-shot generative adaptation) study how to improve the transferability of the generative model on limited data from the target domain. Since we mainly focus on the diffusion model, we summarize the primary GAN-based domain adaptation there. They mainly propose to add different kinds of regularization to avoid model collapse \citep{Ojha2021FewshotIG, Zhao2022ACL, Zhang2022GeneralizedOD, Zhang2022TowardsDA, Duan2023WeditGANFI, Hou2022DynamicWS, Xiao2022FewSG} or finetune subset of the parameter (adaptor) \citep{Alanov2022HyperDomainNetUD, Yang2021OneShotGD, Li2020FewshotIG, Zhao2020OnLP}.

\section{More Discussion on the Potential of the Proposed Method}\label{ap:potential}

In this section, we demonstrate the proposed framework is general enough to deal with text-to-image generation tasks and homogeneous transfer learning.

\subsection{Text-to-Image Generation Tasks}

Given a source distribution $(x, \boldsymbol{c}_t) \sim p$, where $\boldsymbol{c}_t$ denotes text prompts by using the terminology from \cite{Zhang2023AddingCC}, a pre-trained diffusion model can be trained on the source distribution. Given a target distribution $(x, \boldsymbol{c}_t, \boldsymbol{c}_{\mathrm{f}}) \sim q$, where $\boldsymbol{c}_f$ denotes a task-specific condition, Zhang et al. \citep{Zhang2023AddingCC} can fine-tune the pre-trained model by noise matching objective,
$$
\mathcal{L}=\mathbb{E}_{\boldsymbol{x}_0, \boldsymbol{t}, \boldsymbol{c}_t, \boldsymbol{c}_{\mathrm{f}}, \epsilon \sim \mathcal{N}(0,1)}\left[\| \epsilon-\epsilon_\theta\left(\boldsymbol{x}_t, \boldsymbol{t}, \boldsymbol{c}_t, \boldsymbol{c}_{\mathrm{f}}\right) \|_2^2\right].
$$

Our method can directly estimate $\nabla_{\mathbf{x}_t} \log \mathbb{E}_{p(\mathbf{x}_0|\mathbf{x}_t, \boldsymbol{c}_t)}\left[\frac{q(\mathbf{x}_0, \boldsymbol{c}_t, \boldsymbol{c}_{\mathrm{f}})}{p(\mathbf{x}_0, \boldsymbol{c}_t)}\right]$ rather than fine-tuning the pre-trained diffusion model. Domain classifier $c_w(x, y)$ can still be used for estimating the density ratio, where $y$ denotes the embedding of the condition. Moreover, directly fine-tuning the diffusion model on data from the target domain used by \cite{Zhang2023AddingCC} is similar to the consistency regularization proposed in our work, while they have a more in-depth design for the architecture and have great results on vision-language tasks. However, with limited data from the target distribution, direct fine-tuning may not achieve good enough performance, which is verified in the two-dimensional Gaussian setting. In \cite{Zhang2023AddingCC}, Zhang et al. propose to use zero convolution layers, i.e., 1 × 1 convolution layer with both weight and bias initialized to zeros, which alleviates the instability of fine-tuning process. This is very different from our methodology which relies on the smaller sample complexity of the classifier/density ratio estimator.

\subsection{Homogeneous Transfer Learning}
When the source and target domain contain different class labels, our framework is still applicable, i.e., when $y_t\neq y_s$, $$\underbrace{\mathbf{s}_{\boldsymbol{\phi}^*}(\mathbf{x}_t, y_{t}, t)}_{\substack{\text {target source}}} = \underbrace{\nabla_{\mathbf{x}_t} \log p_t(\mathbf{x}_t |  y_s)}_{\substack{\text {pre-trained conditional model}\\ \text{on source}}}+ \underbrace{\nabla_{\mathbf{x}_t} \log \mathbb{E}_{p(\mathbf{x}_0|\mathbf{x}_t, y_s)}\left[\frac{q(\mathbf{x}_0, y_t)}{p(\mathbf{x}_0, y_s)}\right]}_{\text {conditional guidance}}.$$ To generate an unseen class $y_t$, the key problem here is to choose a particular class from the source domain $y_s$ such that we can borrow useful information from the source domain to generate this unseen class from the target domain. The coupling between $y_t$ and $y_s$ can be learned by solving a static optimal transport problem. More in-depth design, e.g. coupling solved by static optimal transport, can be left to future work.

\section{Theoretical Details for Section \ref{sec:rgdp}} \label{app:proofs}

\subsection{Proof of Theorem \ref{thm:IS_guidance}}\label{ap:proof_thm_ISguidance}
\begin{proof}
To prove Eq \eqref{eq:IS_guid}, we first build the connection between Score Matching on the target domain and Importance Weighted Denoising Score Matching on the source domain in the following Lemma.
\begin{lemma}
\label{DSM-equ}
Score Matching on the target domain is equivalent to Importance Weighted Denoising Score Matching on the source domain, i.e.,  
\begin{equation}\label{eq:appendix1}
\begin{aligned}
\boldsymbol{\phi}^*= &\underset{\boldsymbol{\phi}}{\arg \min } \ \mathbb{E}_t\left\{\lambda(t) \mathbb{E}_{q_t(\mathbf{x}_t)} \left[\left\|\mathbf{s}_{\boldsymbol{\phi}}(\mathbf{x}_t, t)-
\nabla_{\mathbf{x}_t} \log q_t(\mathbf{x}_t)\right\|_2^2\right]\right\}\\
= & \underset{\boldsymbol{\phi}}{\arg \min } \ \mathbb{E}_t\left\{\lambda(t) \mathbb{E}_{p(\mathbf{x}_0)} \mathbb{E}_{p(\mathbf{x}_t|\mathbf{x}_0)}\left[\left\|\mathbf{s}_{\boldsymbol{\phi}}(\mathbf{x}_t, t)-
\nabla_{\mathbf{x}_t} \log p(\mathbf{x}_t | \mathbf{x}_0)\right\|_2^2 \frac{q(\mathbf{x}_0)}{p(\mathbf{x}_0)}\right]\right\}.
\end{aligned}    
\end{equation}
\end{lemma}

\begin{proof}[Proof of Lemma \ref{DSM-equ}]
We first connect Score Matching objective in the target domain to Denoising Score Matching objective in target distribution, which is proven by \cite{Vincent2011ACB}, i.e.,
\[
\begin{aligned}
\boldsymbol{\phi}^*= &\underset{\boldsymbol{\phi}}{\arg \min } \ \mathbb{E}_t\left\{\lambda(t) \mathbb{E}_{q_t(\mathbf{x}_t)} \left[\left\|\mathbf{s}_{\boldsymbol{\phi}}(\mathbf{x}_t, t)-
\nabla_{\mathbf{x}_t} \log q_t(\mathbf{x}_t)\right\|_2^2\right]\right\}\\= &\underset{\boldsymbol{\phi}}{\arg \min } \ \mathbb{E}_t\left\{\lambda(t) \mathbb{E}_{q(\mathbf{x}_0)} \mathbb{E}_{q(\mathbf{x}_t|\mathbf{x}_0)}\left[\left\|\mathbf{s}_{\boldsymbol{\phi}}(\mathbf{x}_t, t)-
\nabla_{\mathbf{x}_t} \log q(\mathbf{x}_t | \mathbf{x}_0)\right\|_2^2\right]\right\}.
\end{aligned}
\]
Then, we split the mean squared error of Denoising Score Matching objective on target distribution into three terms as follows:
\begin{align}\label{eq:DSM_target1}
&\mathbb{E}_{q(\mathbf{x}_0)} \mathbb{E}_{q(\mathbf{x}_t|\mathbf{x}_0)}\left[\left\|\mathbf{s}_{\boldsymbol{\phi}}(\mathbf{x}_t, t)-
\nabla_{\mathbf{x}_t} \log q(\mathbf{x}_t | \mathbf{x}_0)\right\|_2^2\right]\notag\\
=& \mathbb{E}_{q(\mathbf{x}_0, \mathbf{x}_t)}\left[\left\|\mathbf{s}_{\boldsymbol{\phi}}(\mathbf{x}_t, t)\right\|_2^2\right] - 2\mathbb{E}_{q(\mathbf{x}_0,\mathbf{x}_t)}\left[\langle \mathbf{s}_{\boldsymbol{\phi}}(\mathbf{x}_t, t), \nabla_{\mathbf{x}_t} \log q(\mathbf{x}_t | \mathbf{x}_0) \rangle\right]  + C_1,
\end{align}
where $C_1=\mathbb{E}_{q(\mathbf{x}_0,\mathbf{x}_t)}\left[\left\| \nabla_{\mathbf{x}_t} \log q(\mathbf{x}_t | \mathbf{x}_0)\right\|_2^2\right]$ is a constant independent with $\boldsymbol{\phi}$. We can similarly split the objective function in the right-hand side (RHS) of Eq \eqref{eq:appendix1} as follows:
\begin{align}\label{eq:ISDSM_source1}
& \mathbb{E}_{p(\mathbf{x}_0)} \mathbb{E}_{p(\mathbf{x}_t|\mathbf{x}_0)}\left[\left\|\mathbf{s}_{\boldsymbol{\phi}}(\mathbf{x}_t, t)-
\nabla_{\mathbf{x}_t} \log p(\mathbf{x}_t | \mathbf{x}_0)\right\|_2^2 \frac{q(\mathbf{x}_0)}{p(\mathbf{x}_0)}\right]\notag\\
=& \mathbb{E}_{p(\mathbf{x}_0, \mathbf{x}_t)}\left[\left\|\mathbf{s}_{\boldsymbol{\phi}}(\mathbf{x}_t, t)\right\|_2^2\frac{q(\mathbf{x}_0)}{p(\mathbf{x}_0)}\right] - 2\mathbb{E}_{p(\mathbf{x}_0,\mathbf{x}_t)}\left[\langle \mathbf{s}_{\boldsymbol{\phi}}(\mathbf{x}_t, t), \nabla_{\mathbf{x}_t} \log p(\mathbf{x}_t | \mathbf{x}_0) \rangle\frac{q(\mathbf{x}_0)}{p(\mathbf{x}_0)}\right] + C_2,
\end{align}
where $C_2$ is a constant independent with $\boldsymbol{\phi}$.
It is easy to show that the first term in Eq \eqref{eq:DSM_target1} is equal to the first term in Eq \eqref{eq:ISDSM_source1}, i.e., 
\[
\begin{aligned}
\mathbb{E}_{p(\mathbf{x}_0, \mathbf{x}_t)}\left[\left\|\mathbf{s}_{\boldsymbol{\phi}}(\mathbf{x}_t, t)\right\|_2^2\frac{q(\mathbf{x}_0)}{p(\mathbf{x}_0)}\right]= &\int_{\mathbf{x}_0} \int_{\mathbf{x}_t} p(\mathbf{x}_0) p(\mathbf{x}_t| \mathbf{x}_0) \left\|\mathbf{s}_{\boldsymbol{\phi}}(\mathbf{x}_t, t)\right\|_2^2\frac{q(\mathbf{x}_0)}{p(\mathbf{x}_0)}d \mathbf{x}_0 d \mathbf{x}_t\\
\overset{(i)}{=} &\int_{\mathbf{x}_0} \int_{\mathbf{x}_t} p(\mathbf{x}_0) q(\mathbf{x}_t| \mathbf{x}_0) \left\|\mathbf{s}_{\boldsymbol{\phi}}(\mathbf{x}_t, t)\right\|_2^2\frac{q(\mathbf{x}_0)}{p(\mathbf{x}_0)}d \mathbf{x}_0 d \mathbf{x}_t\\
= &\int_{\mathbf{x}_0} \int_{\mathbf{x}_t}  q(\mathbf{x}_0,\mathbf{x}_t) \left\|\mathbf{s}_{\boldsymbol{\phi}}(\mathbf{x}_t, t)\right\|_2^2d \mathbf{x}_0 d \mathbf{x}_t\\
= &\mathbb{E}_{q(\mathbf{x}_0, \mathbf{x}_t)}\left[\left\|\mathbf{s}_{\boldsymbol{\phi}}(\mathbf{x}_t, t)\right\|_2^2\right],
\end{aligned}
\]
where the equality $(i)$ is due to $q(\mathbf{x}_t | \mathbf{x}_0)=p(\mathbf{x}_t |\mathbf{x}_0)$.

Next, we prove the second terms in Eq \eqref{eq:DSM_target1} and Eq \eqref{eq:ISDSM_source1} are also equivalent:
\[
\begin{aligned}
&\mathbb{E}_{p(\mathbf{x}_0,\mathbf{x}_t)}\left[\langle \mathbf{s}_{\boldsymbol{\phi}}(\mathbf{x}_t, t), \nabla_{\mathbf{x}_t} \log p(\mathbf{x}_t | \mathbf{x}_0) \rangle\frac{q(\mathbf{x}_0)}{p(\mathbf{x}_0)}\right]\\
=& \int_{\mathbf{x}_0} \int_{\mathbf{x}_t} p(\mathbf{x}_0, \mathbf{x}_t) \langle \mathbf{s}_{\boldsymbol{\phi}}(\mathbf{x}_t, t),  \frac{\nabla_{\mathbf{x}_t} p(\mathbf{x}_t | \mathbf{x}_0)}{p(\mathbf{x}_t | \mathbf{x}_0)}  \rangle\frac{q(\mathbf{x}_0)}{p(\mathbf{x}_0)}d \mathbf{x}_0 d \mathbf{x}_t\\
\overset{(i)}{=}& \int_{\mathbf{x}_0} \int_{\mathbf{x}_t} p(\mathbf{x}_0) p(\mathbf{x}_t| \mathbf{x}_0)  \langle \mathbf{s}_{\boldsymbol{\phi}}(\mathbf{x}_t, t),  \frac{\nabla_{\mathbf{x}_t} q(\mathbf{x}_t | \mathbf{x}_0)}{p(\mathbf{x}_t | \mathbf{x}_0)}  \rangle\frac{q(\mathbf{x}_0)}{p(\mathbf{x}_0)}d \mathbf{x}_0 d \mathbf{x}_t\\
=& \int_{\mathbf{x}_0} \int_{\mathbf{x}_t}  \langle \mathbf{s}_{\boldsymbol{\phi}}(\mathbf{x}_t, t),  \nabla_{\mathbf{x}_t} q(\mathbf{x}_t | \mathbf{x}_0)  \rangle q(\mathbf{x}_0)d \mathbf{x}_0 d \mathbf{x}_t\\
=& \int_{\mathbf{x}_0} \int_{\mathbf{x}_t}  \langle \mathbf{s}_{\boldsymbol{\phi}}(\mathbf{x}_t, t),  \nabla_{\mathbf{x}_t} \log q(\mathbf{x}_t | \mathbf{x}_0)  \rangle q(\mathbf{x}_t | \mathbf{x}_0) q(\mathbf{x}_0)d \mathbf{x}_0 d \mathbf{x}_t\\
=&\mathbb{E}_{q(\mathbf{x}_0,\mathbf{x}_t)}\left[\langle \mathbf{s}_{\boldsymbol{\phi}}(\mathbf{x}_t, t), \nabla_{\mathbf{x}_t} \log q(\mathbf{x}_t | \mathbf{x}_0) \rangle\right],
\end{aligned}
\]
where the equality $(i)$ is again due to $q(\mathbf{x}_t | \mathbf{x}_0)=p(\mathbf{x}_t |\mathbf{x}_0)$. Thereby we prove Eq \ref{eq:appendix1}.
\end{proof}

According to Lemma \ref{DSM-equ}, 
\[
\boldsymbol{\phi}^*= \underset{\boldsymbol{\phi}}{\arg \min } ~\mathbb{E}_t\left\{\lambda(t) \mathbb{E}_{p(\mathbf{x}_0)} \mathbb{E}_{p(\mathbf{x}_t|\mathbf{x}_0)}\left[\left\|\mathbf{s}_{\boldsymbol{\phi}}(\mathbf{x}_t, t)-
\nabla_{\mathbf{x}_t} \log p(\mathbf{x}_t | \mathbf{x}_0)\right\|_2^2 \frac{q(\mathbf{x}_0)}{p(\mathbf{x}_0)}\right]\right\}.
\]
Based on this, we may use Importance Weighted Denoising Score Matching on the source domain to get the analytic form of $\mathbf{s}_{\boldsymbol{\phi}^*}$ as follows: 
\[
\mathbf{s}_{\boldsymbol{\phi}^*}(\mathbf{x}_t, t) = \frac{\mathbb{E}_{p(\mathbf{x}_0|\mathbf{x}_t)}\left[\nabla_{\mathbf{x}_t} \log p(\mathbf{x}_t | \mathbf{x}_0)\frac{q(\mathbf{x}_0)}{p(\mathbf{x}_0)}   \right]}{\mathbb{E}_{p(\mathbf{x}_0|\mathbf{x}_t)}\left[\frac{q(\mathbf{x}_0)}{p(\mathbf{x}_0)}   \right]}.
\]
The RHS of Eq \eqref{eq:IS_guid} can be rewritten as follows:
\[
\begin{aligned}
\text{RHS} =& \nabla_{\mathbf{x}_t} \log p_t(\mathbf{x}_t)+ \nabla_{\mathbf{x}_t} \log \mathbb{E}_{p(\mathbf{x}_0|\mathbf{x}_t)}\left[\frac{q(\mathbf{x}_0)}{p(\mathbf{x}_0)}\right]= \nabla_{\mathbf{x}_t} \log p_t(\mathbf{x}_t)+ \frac{\nabla_{\mathbf{x}_t}  \mathbb{E}_{p(\mathbf{x}_0|\mathbf{x}_t)}\left[\frac{q(\mathbf{x}_0)}{p(\mathbf{x}_0)}\right]}{\mathbb{E}_{p(\mathbf{x}_0|\mathbf{x}_t)}\left[\frac{q(\mathbf{x}_0)}{p(\mathbf{x}_0)}\right]}\\
=& \nabla_{\mathbf{x}_t} \log p_t(\mathbf{x}_t)+ \frac{  \mathbb{E}_{p(\mathbf{x}_0|\mathbf{x}_t)}\left[\frac{q(\mathbf{x}_0)}{p(\mathbf{x}_0)}\nabla_{\mathbf{x}_t} \log p(\mathbf{x}_0|\mathbf{x}_t)\right]}{\mathbb{E}_{p(\mathbf{x}_0|\mathbf{x}_t)}\left[\frac{q(\mathbf{x}_0)}{p(\mathbf{x}_0)}\right]}.
\end{aligned}
\]
Since 
\[
\begin{aligned}
\nabla_{\mathbf{x}_t} \log p(\mathbf{x}_0|\mathbf{x}_t) &= \nabla_{\mathbf{x}_t} \log p(\mathbf{x}_t|\mathbf{x}_0)+ \nabla_{\mathbf{x}_t} \log p(\mathbf{x}_0) - \nabla_{\mathbf{x}_t} \log p_t(\mathbf{x}_t)\\
&= \nabla_{\mathbf{x}_t} \log p(\mathbf{x}_t|\mathbf{x}_0) - \nabla_{\mathbf{x}_t} \log p_t(\mathbf{x}_t),
\end{aligned}
\]
we can further rewrite the RHS of Eq \eqref{eq:IS_guid} as follows:
\[
\begin{aligned}
\text{RHS} =& \nabla_{\mathbf{x}_t} \log p_t(\mathbf{x}_t)+ \frac{  \mathbb{E}_{p(\mathbf{x}_0|\mathbf{x}_t)}\left[\frac{q(\mathbf{x}_0)}{p(\mathbf{x}_0)}\nabla_{\mathbf{x}_t} \log p(\mathbf{x}_t|\mathbf{x}_0)\right]}{\mathbb{E}_{p(\mathbf{x}_0|\mathbf{x}_t)}\left[\frac{q(\mathbf{x}_0)}{p(\mathbf{x}_0)}\right]} -\nabla_{\mathbf{x}_t} \log p_t(\mathbf{x}_t)\\
=&\frac{\mathbb{E}_{p(\mathbf{x}_0|\mathbf{x}_t)}\left[\nabla_{\mathbf{x}_t} \log p(\mathbf{x}_t | \mathbf{x}_0)\frac{q(\mathbf{x}_0)}{p(\mathbf{x}_0)}   \right]}{\mathbb{E}_{p(\mathbf{x}_0|\mathbf{x}_t)}\left[\frac{q(\mathbf{x}_0)}{p(\mathbf{x}_0)}   \right]}\\
=& \mathbf{s}_{\boldsymbol{\phi}^*}(\mathbf{x}_t, t).
\end{aligned}
\]
Thereby we complete the proof.
\end{proof}

\subsection{Proof of Lemma \ref{thm:exact_guidance}}
\begin{proof}
The proof is straightforward and we include it below for completeness. Note that the objective function can be rewritten as 
\[
\begin{aligned}
    \mathcal{L}_{\text{guidance}}(\boldsymbol{\psi}) := & \mathbb{E}_{p(\mathbf{x}_0, \mathbf{x}_t)}\left[\left\|h_{\boldsymbol{\psi}}\left(\mathbf{x}_t, t\right)-\frac{q(\mathbf{x}_0)}{p(\mathbf{x}_0)}\right\|_2^2\right]\\
    = &   \int_{\mathbf{x}_t} \left\{\int_{\mathbf{x}_0}p(\mathbf{x}_0|\mathbf{x}_t) \left\|h_{\boldsymbol{\psi}}\left(\mathbf{x}_t, t\right)-\frac{q(\mathbf{x}_0)}{p(\mathbf{x}_0)}\right\|_2^2 d\mathbf{x}_0 \right\} p(\mathbf{x}_t)d\mathbf{x}_t \\
    = & \int_{\mathbf{x}_t} \left\{ \left\|h_{\boldsymbol{\psi}}(\mathbf{x}_t, t)\right\|_2^2  -  2 \langle h_{\boldsymbol{\psi}}(\mathbf{x}_t, t),  \int_{\mathbf{x}_0}p(\mathbf{x}_0|\mathbf{x}_t)  \frac{q(\mathbf{x}_0)}{p(\mathbf{x}_0)}    d\mathbf{x}_0 \rangle \right\} p(\mathbf{x}_t)d\mathbf{x}_t + C \\
    = &  \int_{\mathbf{x}_t}  \left\|h_{\boldsymbol{\psi}}(\mathbf{x}_t, t) - \mathbb{E}_{p(\mathbf{x}_0 |\mathbf{x}_t)}\left[\frac{q(\mathbf{x}_0)}{p(\mathbf{x}_0)}\right] 
    \right\|_2^2 p(\mathbf{x}_t)d\mathbf{x}_t,
\end{aligned}
\]
where $C$ is a constant independent of $\boldsymbol{\psi}$. Thus we have the minimizer $\boldsymbol{\psi}^* = \underset{\boldsymbol{\psi}}{\arg \min } \ \mathcal{L}_{\text{guidance}}(\boldsymbol{\psi})$ satisfies $h_{\boldsymbol{\psi}^*}\left(\mathbf{x}_t, t\right)=\mathbb{E}_{p(\mathbf{x}_0|\mathbf{x}_t)}\left[{q(\mathbf{x}_0)}/{p(\mathbf{x}_0)}\right]$.
\end{proof}

\subsection{Conditional version for Lemma \ref{thm:exact_guidance}}
\label{ap:conditional_exact_guidance}

\begin{lemma}\label{thm:exact_guidance_condi}
For a neural network $h_{\boldsymbol{\psi}}\left(\mathbf{x}_t, y, t\right)$ parameterized by $\boldsymbol{\psi}$, define the objective  
\begin{equation} \label{eq:guidance}
\mathcal{L}_{\text{guidance}}(\boldsymbol{\psi}) :=\mathbb{E}_{p(\mathbf{x}_0, \mathbf{x}_t, y)}\left[\left\|h_{\boldsymbol{\psi}}\left(\mathbf{x}_t, y, t\right)-\frac{q(\mathbf{x}_0, y)}{p(\mathbf{x}_0, y)}\right\|_2^2\right],
\end{equation}
then its minimizer $\boldsymbol{\psi}^* = \underset{\boldsymbol{\psi}}{\arg \min } \ \mathcal{L}_{\text{guidance}}(\boldsymbol{\psi})$ satisfies:
\[
h_{\boldsymbol{\psi}^*}\left(\mathbf{x}_t, y, t\right)=\mathbb{E}_{p(\mathbf{x}_0 |\mathbf{x}_t, y)}\left[{q(\mathbf{x}_0, y)}/{p(\mathbf{x}_0, y)}\right].
\]
\end{lemma} 

\paragraph{Proof of Lemma \ref{thm:exact_guidance_condi}}
\begin{proof}

The proof is straightforward and we include it below for completeness. Note that the objective function can be rewritten as 
\[
\begin{aligned}
    &\mathcal{L}_{\text{guidance}}(\boldsymbol{\psi})    \\
    := & \mathbb{E}_{p(\mathbf{x}_0, \mathbf{x}_t, y)}\left[\left\|h_{\boldsymbol{\psi}}\left(\mathbf{x}_t, y, t\right)-\frac{q(\mathbf{x}_0, y)}{p(\mathbf{x}_0, y)}\right\|_2^2\right]\\
    = &   \int_{\mathbf{x}_t} \int_{y} \left\{\int_{\mathbf{x}_0}p(\mathbf{x}_0|\mathbf{x}_t,y) \left\|h_{\boldsymbol{\psi}}\left(\mathbf{x}_t, y, t\right)-\frac{q(\mathbf{x}_0, y)}{p(\mathbf{x}_0, y)}\right\|_2^2 d\mathbf{x}_0 \right\} p(\mathbf{x}_t|y) p(y) dy d\mathbf{x}_t \\
    = & \int_{\mathbf{x}_t} \int_{y} \left\{ \left\|h_{\boldsymbol{\psi}}(\mathbf{x}_t, y, t)\right\|_2^2  -  2 \langle h_{\boldsymbol{\psi}}(\mathbf{x}_t, y, t),  \int_{\mathbf{x}_0}p(\mathbf{x}_0|\mathbf{x}_t, y)  \frac{q(\mathbf{x}_0,y)}{p(\mathbf{x}_0,y)}    d\mathbf{x}_0 \rangle \right\} p(\mathbf{x}_t|y) p(y) dyd\mathbf{x}_t + C \\
    = &  \int_{\mathbf{x}_t} \int_{y} \left\|h_{\boldsymbol{\psi}}(\mathbf{x}_t, y, t) - \mathbb{E}_{p(\mathbf{x}_0 |\mathbf{x}_t, y)}\left[\frac{q(\mathbf{x}_0, y)}{p(\mathbf{x}_0, y)}\right] 
    \right\|_2^2 p(\mathbf{x}_t|y) p(y) dyd\mathbf{x}_t,
\end{aligned}
\]
where $C$ is a constant independent of $\boldsymbol{\psi}$. Thus we have the minimizer $\boldsymbol{\psi}^* = \underset{\boldsymbol{\psi}}{\arg \min } \ \mathcal{L}_{\text{guidance}}(\boldsymbol{\psi})$ satisfies $h_{\boldsymbol{\psi}^*}\left(\mathbf{x}_t, y, t\right)=\mathbb{E}_{p(\mathbf{x}_0|\mathbf{x}_t, y)}\left[{q(\mathbf{x}_0, y)}/{p(\mathbf{x}_0, y)}\right]$.
\end{proof}

\subsection{Proof of Theorem \ref{thm:IS_guidance_conditional}}\label{ap:proof_thm_ISguidance_conditional}

\begin{proof}
The proof is similar to the proof of Theorem \ref{thm:IS_guidance}. To prove Eq \ref{eq:guide_conditional}, we first build the connection between the Conditional Score Matching on the target domain and Importance Weighted Conditional Denoising Score Matching on the source domain in the following Lemma:

\begin{lemma}
\label{DSM-equ-guidance}
Conditional Score Matching on the target domain is equivalent to Importance Weighted Denoising Score Matching on the source domain, i.e.,  
\[
\begin{aligned}
\boldsymbol{\phi}^*= &\underset{\boldsymbol{\phi}}{\arg \min } \ \mathbb{E}_t\left\{\lambda(t) \mathbb{E}_{q_t(\mathbf{x}_t, y)} \left[\left\|\mathbf{s}_{\boldsymbol{\phi}}(\mathbf{x}_t, y, t)-
\nabla_{\mathbf{x}_t} \log q_t(\mathbf{x}_t | y)\right\|_2^2\right]\right\}\\
= & \underset{\boldsymbol{\phi}}{\arg \min } \ \mathbb{E}_t\left\{\lambda(t) \mathbb{E}_{p(\mathbf{x}_0, y)} \mathbb{E}_{p(\mathbf{x}_t|\mathbf{x}_0)}\left[\left\|\mathbf{s}_{\boldsymbol{\phi}}(\mathbf{x}_t, y,  t)-
\nabla_{\mathbf{x}_t} \log p(\mathbf{x}_t | \mathbf{x}_0)\right\|_2^2 \frac{q(\mathbf{x}_0, y)}{p(\mathbf{x}_0, y)}\right]\right\}.
\end{aligned}
\]
\end{lemma}

\begin{proof}[Proof of Lemma \ref{DSM-equ-guidance}]
We first connect the Conditional Score Matching objective in the target domain to the Conditional Denoising Score Matching objective in target distribution, which is proven by \cite[Theorem 1]{Batzolis2021ConditionalIG}, i.e.,
\[
\begin{aligned}
\boldsymbol{\phi}^*= &\underset{\boldsymbol{\phi}}{\arg \min } \ \mathbb{E}_t\left\{\lambda(t) \mathbb{E}_{q_t(\mathbf{x}_t, y)} \left[\left\|\mathbf{s}_{\boldsymbol{\phi}}(\mathbf{x}_t, y, t)-
\nabla_{\mathbf{x}_t} \log q_t(\mathbf{x}_t|y)\right\|_2^2\right]\right\}\\= &\underset{\boldsymbol{\phi}}{\arg \min } \ \mathbb{E}_t\left\{\lambda(t) \mathbb{E}_{q(\mathbf{x}_0,y)} \mathbb{E}_{q(\mathbf{x}_t|\mathbf{x}_0)}\left[\left\|\mathbf{s}_{\boldsymbol{\phi}}(\mathbf{x}_t, y, t)-
\nabla_{\mathbf{x}_t} \log q(\mathbf{x}_t | \mathbf{x}_0)\right\|_2^2\right]\right\}.
\end{aligned}
\]

Then we split the mean squared error of the Conditional Denoising Score Matching objective on target distribution into three terms as follows:
\begin{align}\label{eq:DSM_target}
&\mathbb{E}_{q(\mathbf{x}_0, y)} \mathbb{E}_{q(\mathbf{x}_t|\mathbf{x}_0)}\left[\left\|\mathbf{s}_{\boldsymbol{\phi}}(\mathbf{x}_t, y, t)-
\nabla_{\mathbf{x}_t} \log q(\mathbf{x}_t | \mathbf{x}_0)\right\|_2^2\right]\notag\\
=& \mathbb{E}_{q(\mathbf{x}_0, \mathbf{x}_t, y)}\left[\left\|\mathbf{s}_{\boldsymbol{\phi}}(\mathbf{x}_t, y, t)\right\|_2^2\right] - 2\mathbb{E}_{q(\mathbf{x}_0,\mathbf{x}_t, y)}\left[\langle \mathbf{s}_{\boldsymbol{\phi}}(\mathbf{x}_t, y, t), \nabla_{\mathbf{x}_t} \log q(\mathbf{x}_t | \mathbf{x}_0) \rangle\right]  + C_1,
\end{align}
where $C_1=\mathbb{E}_{q(\mathbf{x}_0,\mathbf{x}_t, y)}\left[\left\| \nabla_{\mathbf{x}_t} \log q(\mathbf{x}_t | \mathbf{x}_0)\right\|_2^2\right]$ is a constant independent with $\boldsymbol{\phi}$, and $q(\mathbf{x}_t|\mathbf{x}_0,y)=q(\mathbf{x}_t|\mathbf{x}_0) $ because of conditional independent of $\mathbf{x}_t$ and $y$
given $\mathbf{x}_0$ by assumption. We can similarly split the mean squared error of Denoising Score Matching on the source domain into three terms as follows:
\begin{equation}\label{eq:ISDSM_source}
\begin{aligned}
& \mathbb{E}_{p(\mathbf{x}_0, y)} \mathbb{E}_{p(\mathbf{x}_t|\mathbf{x}_0)}\left[\left\|\mathbf{s}_{\boldsymbol{\phi}}(\mathbf{x}_t, y, t)-
\nabla_{\mathbf{x}_t} \log p(\mathbf{x}_t | \mathbf{x}_0)\right\|_2^2 \frac{q(\mathbf{x}_0, y)}{p(\mathbf{x}_0, y)}\right]\\
=& \mathbb{E}_{p(\mathbf{x}_0, \mathbf{x}_t, y)}\left[\left\|\mathbf{s}_{\boldsymbol{\phi}}(\mathbf{x}_t, y, t)\right\|_2^2\frac{q(\mathbf{x}_0, y)}{p(\mathbf{x}_0, y)}\right] - 2\mathbb{E}_{p(\mathbf{x}_0,\mathbf{x}_t, y)}\left[\langle \mathbf{s}_{\boldsymbol{\phi}}(\mathbf{x}_t, y, t), \nabla_{\mathbf{x}_t} \log p(\mathbf{x}_t | \mathbf{x}_0) \rangle\frac{q(\mathbf{x}_0, y)}{p(\mathbf{x}_0, y)}\right] \\
&+ C_2,
\end{aligned}
\end{equation}
where $C_2$ is a constant independent with $\boldsymbol{\phi}$.

It is obvious to show that the first term in Eq \eqref{eq:DSM_target} is equal to the first term in Eq \eqref{eq:ISDSM_source}, i.e., 
\[
\begin{aligned}
&\mathbb{E}_{p(\mathbf{x}_0, \mathbf{x}_t, y)}\left[\left\|\mathbf{s}_{\boldsymbol{\phi}}(\mathbf{x}_t, y, t)\right\|_2^2\frac{q(\mathbf{x}_0, y)}{p(\mathbf{x}_0, y)}\right]\\
= &\int_{\mathbf{x}_0} \int_{\mathbf{x}_t} \int_y p(\mathbf{x}_0, y) p(\mathbf{x}_t| \mathbf{x}_0) \left\|\mathbf{s}_{\boldsymbol{\phi}}(\mathbf{x}_t, y, t)\right\|_2^2\frac{q(\mathbf{x}_0, y)}{p(\mathbf{x}_0, y)}d \mathbf{x}_0 d \mathbf{x}_t d y\\
= &\int_{\mathbf{x}_0} \int_{\mathbf{x}_t} \int_y p(\mathbf{x}_0, y) q(\mathbf{x}_t| \mathbf{x}_0) \left\|\mathbf{s}_{\boldsymbol{\phi}}(\mathbf{x}_t, y, t)\right\|_2^2\frac{q(\mathbf{x}_0, y)}{p(\mathbf{x}_0, y)}d \mathbf{x}_0 d \mathbf{x}_t d y\\
= &\int_{\mathbf{x}_0} \int_{\mathbf{x}_t} \int_{y} q(\mathbf{x}_0,\mathbf{x}_t, y) \left\|\mathbf{s}_{\boldsymbol{\phi}}(\mathbf{x}_t, y, t)\right\|_2^2d \mathbf{x}_0 d \mathbf{x}_t d y\\
= &\mathbb{E}_{q(\mathbf{x}_0, \mathbf{x}_t, y)}\left[\left\|\mathbf{s}_{\boldsymbol{\phi}}(\mathbf{x}_t, y, t)\right\|_2^2\right].
\end{aligned}
\]
And the second term is also equivalent:
\[
\begin{aligned}
&\mathbb{E}_{p(\mathbf{x}_0,\mathbf{x}_t, y)}\left[\langle \mathbf{s}_{\boldsymbol{\phi}}(\mathbf{x}_t, y, t), \nabla_{\mathbf{x}_t} \log p(\mathbf{x}_t | \mathbf{x}_0) \rangle\frac{q(\mathbf{x}_0, y)}{p(\mathbf{x}_0, y)}\right]\\
=& \int_{\mathbf{x}_0} \int_{\mathbf{x}_t} \int_y p(\mathbf{x}_0, \mathbf{x}_t, y) \langle \mathbf{s}_{\boldsymbol{\phi}}(\mathbf{x}_t, y, t),  \frac{\nabla_{\mathbf{x}_t} p(\mathbf{x}_t | \mathbf{x}_0)}{p(\mathbf{x}_t | \mathbf{x}_0)}  \rangle\frac{q(\mathbf{x}_0, y)}{p(\mathbf{x}_0, y)}d \mathbf{x}_0 d \mathbf{x}_t d y\\
=& \int_{\mathbf{x}_0} \int_{\mathbf{x}_t} \int_y p(\mathbf{x}_0, \mathbf{x}_t, y) \langle \mathbf{s}_{\boldsymbol{\phi}}(\mathbf{x}_t, y, t),  \frac{\nabla_{\mathbf{x}_t} q(\mathbf{x}_t | \mathbf{x}_0)}{p(\mathbf{x}_t | \mathbf{x}_0)}  \rangle\frac{q(\mathbf{x}_0, y)}{p(\mathbf{x}_0, y)}d \mathbf{x}_0 d \mathbf{x}_t d y\\
=& \int_{\mathbf{x}_0} \int_{\mathbf{x}_t} \int_y \langle \mathbf{s}_{\boldsymbol{\phi}}(\mathbf{x}_t, y, t),  \nabla_{\mathbf{x}_t} q(\mathbf{x}_t | \mathbf{x}_0)  \rangle q(\mathbf{x}_0, y)d \mathbf{x}_0 d \mathbf{x}_t d y\\
=& \int_{\mathbf{x}_0} \int_{\mathbf{x}_t} \int_y \langle \mathbf{s}_{\boldsymbol{\phi}}(\mathbf{x}_t, y, t),  \nabla_{\mathbf{x}_t} \log q(\mathbf{x}_t | \mathbf{x}_0)  \rangle q(\mathbf{x}_t | \mathbf{x}_0) q(\mathbf{x}_0, y)d \mathbf{x}_0 d \mathbf{x}_t d y\\
=&\mathbb{E}_{q(\mathbf{x}_0,\mathbf{x}_t, y)}\left[\langle \mathbf{s}_{\boldsymbol{\phi}}(\mathbf{x}_t, y, t), \nabla_{\mathbf{x}_t} \log q(\mathbf{x}_t | \mathbf{x}_0) \rangle\right].
\end{aligned}
\]
\end{proof}

According to Lemma \ref{DSM-equ-guidance}, the optimal solution satisfies
\[
\boldsymbol{\phi}^*= \underset{\boldsymbol{\phi}}{\arg \min } ~\mathbb{E}_t\left\{\lambda(t) \mathbb{E}_{p(\mathbf{x}_0, y)} \mathbb{E}_{p(\mathbf{x}_t|\mathbf{x}_0)}\left[\left\|\mathbf{s}_{\boldsymbol{\phi}}(\mathbf{x}_t, y,  t)-
\nabla_{\mathbf{x}_t} \log p(\mathbf{x}_t | \mathbf{x}_0)\right\|_2^2 \frac{q(\mathbf{x}_0, y)}{p(\mathbf{x}_0, y)}\right]\right\},
\]
Then, we use Importance Weighted Conditional Denoising Score Matching on the source domain to get the analytic form of $\mathbf{s}_{\boldsymbol{\phi}^*}$ as follows: 
\[
\mathbf{s}_{\boldsymbol{\phi}^*}(\mathbf{x}_t, y, t) = \frac{\mathbb{E}_{p(\mathbf{x}_0|\mathbf{x}_t, y)}\left[\nabla_{\mathbf{x}_t} \log p(\mathbf{x}_t | \mathbf{x}_0)\frac{q(\mathbf{x}_0, y)}{p(\mathbf{x}_0, y)}   \right]}{\mathbb{E}_{p(\mathbf{x}_0|\mathbf{x}_t, y)}\left[\frac{q(\mathbf{x}_0, y)}{p(\mathbf{x}_0, y)}   \right]}.
\]

Moreover, the RHS of Eq \eqref{eq:guide_conditional} can be rewritten as:
\[
\begin{aligned}
\text{RHS} =& \nabla_{\mathbf{x}_t} \log p_t(\mathbf{x}_t|y)+ \nabla_{\mathbf{x}_t} \log \mathbb{E}_{p(\mathbf{x}_0|\mathbf{x}_t, y)}\left[\frac{q(\mathbf{x}_0, y)}{p(\mathbf{x}_0, y)}\right]\\
=& \nabla_{\mathbf{x}_t} \log p_t(\mathbf{x}_t|y)+ \frac{\nabla_{\mathbf{x}_t}  \mathbb{E}_{p(\mathbf{x}_0|\mathbf{x}_t, y)}\left[\frac{q(\mathbf{x}_0, y)}{p(\mathbf{x}_0, y)}\right]}{\mathbb{E}_{p(\mathbf{x}_0|\mathbf{x}_t, y)}\left[\frac{q(\mathbf{x}_0, y)}{p(\mathbf{x}_0, y)}\right]}\\
=& \nabla_{\mathbf{x}_t} \log p_t(\mathbf{x}_t|y)+ \frac{  \mathbb{E}_{p(\mathbf{x}_0|\mathbf{x}_t, y)}\left[\frac{q(\mathbf{x}_0, y)}{p(\mathbf{x}_0, y)}\nabla_{\mathbf{x}_t} \log p(\mathbf{x}_0|\mathbf{x}_t, y)\right]}{\mathbb{E}_{p(\mathbf{x}_0|\mathbf{x}_t, y)}\left[\frac{q(\mathbf{x}_0, y)}{p(\mathbf{x}_0, y)}\right]}.
\end{aligned}
\]
Since 
\[
\begin{aligned} 
\nabla_{\mathbf{x}_t} \log p(\mathbf{x}_0|\mathbf{x}_t, y)&=\nabla_{\mathbf{x}_t} \log p(\mathbf{x}_t|\mathbf{x}_0, y)+ \nabla_{\mathbf{x}_t} \log p(\mathbf{x}_0|y) - \nabla_{\mathbf{x}_t} \log p_t(\mathbf{x}_t|y)\\
&= \nabla_{\mathbf{x}_t} \log p(\mathbf{x}_ t|\mathbf{x}_0, y) - \nabla_{\mathbf{x}_t} \log p_t(\mathbf{x}_t|y),\\
&= \nabla_{\mathbf{x}_t} \log p(\mathbf{x}_ t|\mathbf{x}_0) - \nabla_{\mathbf{x}_t} \log p_t(\mathbf{x}_t|y),
\end{aligned}
\]
we can further simplify the RHS of Eq \eqref{eq:guide_conditional} as follows:
\[
\begin{aligned}
\text{RHS} =& \nabla_{\mathbf{x}_t} \log p_t(\mathbf{x}_t|y)+ \frac{  \mathbb{E}_{p(\mathbf{x}_0|\mathbf{x}_t, y)}\left[\frac{q(\mathbf{x}_0, y)}{p(\mathbf{x}_0, y)}\nabla_{\mathbf{x}_t} \log p(\mathbf{x}_t|\mathbf{x}_0)\right]}{\mathbb{E}_{p(\mathbf{x}_0|\mathbf{x}_t, y)}\left[\frac{q(\mathbf{x}_0, y)}{p(\mathbf{x}_0, y)}\right]} -\nabla_{\mathbf{x}_t} \log p_t(\mathbf{x}_t|y)\\
=&\frac{\mathbb{E}_{p(\mathbf{x}_0|\mathbf{x}_t, y)}\left[\nabla_{\mathbf{x}_t} \log p(\mathbf{x}_t | \mathbf{x}_0)\frac{q(\mathbf{x}_0, y)}{p(\mathbf{x}_0, y)}   \right]}{\mathbb{E}_{p(\mathbf{x}_0|\mathbf{x}_t, y)}\left[\frac{q(\mathbf{x}_0, y)}{p(\mathbf{x}_0, y)}   \right]}\\
=& \mathbf{s}_{\boldsymbol{\phi}^*}(\mathbf{x}_t, t).
\end{aligned}
\]
Thereby, we finish the proof.
\end{proof}

\subsection{Proof for Cycle Regularization}\label{ap:proof_cycle_regularization}
\begin{proof}[Proof of Eq \eqref{eq:cycle_reg}]
$$
\begin{aligned}
\mathbb{E}_{p(\mathbf{x}_0|\mathbf{x}_t)}\left[\frac{q(\mathbf{x}_0)}{p(\mathbf{x}_0)}\right]
=&\int p(\mathbf{x}_0 | \mathbf{x}_t) \frac{q(\mathbf{x}_0)}{p(\mathbf{x}_0)} d \mathbf{x}_0=\int \frac{p(\mathbf{x}_t | \mathbf{x}_0) p(\mathbf{x}_0)}{p_t(\mathbf{x}_t)}\frac{q(\mathbf{x}_0)}{p(\mathbf{x}_0)} d \mathbf{x}_0\\
=& \int \frac{q(\mathbf{x}_t | \mathbf{x}_0) p(\mathbf{x}_0)}{p_t(\mathbf{x}_t)}\frac{q(\mathbf{x}_0)}{p(\mathbf{x}_0)} d \mathbf{x}_0= \int q(\mathbf{x}_t | \mathbf{x}_0)\frac{q(\mathbf{x}_0)}{p_t(\mathbf{x}_t)} d \mathbf{x}_0\\
=& \int \frac{q(\mathbf{x}_0 | \mathbf{x}_t) q_t(\mathbf{x}_t)}{q(\mathbf{x}_0)}\frac{q(\mathbf{x}_0)}{p_t(\mathbf{x}_t)} d \mathbf{x}_0=\int q(\mathbf{x}_0 | \mathbf{x}_t) \frac{q_t(\mathbf{x}_t)}{p_t(\mathbf{x}_t)} d \mathbf{x}_0\\
=&\mathbb{E}_{q(\mathbf{x}_0 | \mathbf{x}_t)}\left[\frac{q_t(\mathbf{x}_t)}{p_t(\mathbf{x}_t)}\right],
\end{aligned}
$$
where $p_t(\mathbf{x}_t) = \int p(\mathbf{x}_0) p(\mathbf{x}_t| \mathbf{x}_0) d\mathbf{x}_0 $ and $q_t(\mathbf{x}_t) = \int q(\mathbf{x}_0) q(\mathbf{x}_t|\mathbf{x}_0) d\mathbf{x}_0 $ are the marginal distributions at time $t$ of source and target distributions, respectively.
\end{proof}


\section{More Details on Experiments}

\subsection{Algorithms for TGDP}\label{app:alg_box}

TGDP adopts Algorithm \ref{alg:DC_training} and \ref{alg:DC_training_time} for training a domain classifier and Algorithm \ref{alg:guidance_training} and \ref{alg:guidance_training_regularized} for training the guidance network.

\begin{algorithm}[htb]
    \centering
    \caption{Training a domain classifier}\label{alg:DC_training}
\begin{algorithmic}[1]
      \REQUIRE Samples from the marginal distribution of the source domain $p(\mathbf{x})$ and target domain $q(\mathbf{x})$, and initial weights of domain classifier $\boldsymbol{\omega}$.
      \REPEAT 
      \STATE Sample mini-batch data from source distribution and target distribution respectively with batch size $b$.
    \STATE Take gradient descent step on
$$
\nabla_{\boldsymbol{\omega}}\left\{-\frac{1}{b}\sum_{\mathbf{x}_i\in p}\left[\log c_{\boldsymbol{\omega}}(\mathbf{x}_i)\right]-\frac{1}{b}\sum_{\mathbf{x}'_i\in q}\left[\log (1-c_{\boldsymbol{\omega}}(\mathbf{x}'_i))\right] \right\}.
$$
\UNTIL{converged.}
    \STATE \textbf{return} weights of domain classifier  $\boldsymbol{\omega}$.
	\end{algorithmic}
\end{algorithm}

\begin{algorithm}[htb]
    \centering
    \caption{Training a time-dependent domain classifier}\label{alg:DC_training_time}
\begin{algorithmic}[1]
      \REQUIRE Samples from the marginal distribution of the source domain $p(\mathbf{x})$ and target domain $q(\mathbf{x})$, pre-defined forward transition $p(\mathbf{x}_t |\mathbf{x}_0)$, and initial weights of domain classifier $\boldsymbol{\omega}$.
      \REPEAT 
      \STATE Sample mini-batch data from source distribution and target distribution respectively with batch size $b$.
    \STATE Sample time $t \sim \operatorname{Uniform}(\{1, \ldots, T\})$ and perturb $\mathbf{x}_0$ by forward transition $p(\mathbf{x}_t |\mathbf{x}_0)$.
    \STATE Take gradient descent step on
$$
\nabla_{\boldsymbol{\omega}}\left\{-\frac{1}{b}\sum_{\mathbf{x}_0\sim p}\sum_{\mathbf{x}_t |\mathbf{x}_0}\left[\log c_{\boldsymbol{\omega}}(\mathbf{x}_t,t)\right] -\frac{1}{b}\sum_{\mathbf{x}_0\sim q}\sum_{\mathbf{x}_t |\mathbf{x}_0}\left[\log (1-c_{\boldsymbol{\omega}}(\mathbf{x}_t,t))\right] \right\}.
$$
\UNTIL{converged.}
    \STATE \textbf{return} weights of time-dependent domain classifier  $\boldsymbol{\omega}$.
	\end{algorithmic}
\end{algorithm}

\begin{algorithm}[htb]
    \centering
    \caption{Training a guidance network (without regularization)}\label{alg:guidance_training}
\begin{algorithmic}[1]
      \REQUIRE Samples from the marginal distribution of the source domain $p(\mathbf{x})$, pre-defined forward transition $p(\mathbf{x}_t |\mathbf{x}_0)$, pre-trained domain classifier $c_{\boldsymbol{\omega}}$, and initial weights of guidance network  $\boldsymbol{\psi}$.
      \REPEAT
      \STATE Sample mini-batch data from source distribution $\mathbf{x}_0$ with batch size $b$.
      \STATE Sample time $t \sim \operatorname{Uniform}(\{1, \ldots, T\})$ and perturb $\mathbf{x}_0$ by forward transition $p(\mathbf{x}_t |\mathbf{x}_0)$.
\STATE Take gradient descent step on
$$
\nabla_{\boldsymbol{\psi}}\left\{\frac{1}{b}\sum\limits_{\mathbf{x}_0, \mathbf{x}_t}\left[\left\|h_{\boldsymbol{\psi}}\left(\mathbf{x}_t, t\right)-c_{\boldsymbol{\omega}}(\mathbf{x}_0)\right\|_2^2\right]\right\}.
$$
\UNTIL{converged.}
    \STATE \textbf{return} weights of guidance network  $\boldsymbol{\psi}$.
	\end{algorithmic}
\end{algorithm}

\begin{algorithm}[htb]
    \centering
    \caption{Training a guidance network (with regularization)}\label{alg:guidance_training_regularized}
\begin{algorithmic}[1]
      \REQUIRE Samples from the marginal distribution of the source domain $p(\mathbf{x})$ and target domain $q(\mathbf{x})$, pre-trained diffusion model on source distribution $s_{\text{source}}(\mathbf{x}_t, t)$, pre-defined forward transition $q(\mathbf{x}_t | \mathbf{x}_0), p(\mathbf{x}_t |\mathbf{x}_0)$, pre-trained domain classifier $c_{\boldsymbol{\omega}}(\mathbf{x}_0)$ and time dependent domain classifier $c'_{\boldsymbol{\omega}}(\mathbf{x}_0, t)$, hyperparameter $\eta_1, \eta_2$, and initial weights of guidance network  $\boldsymbol{\psi}$.
      \REPEAT
      \STATE Sample mini-batch data from source distribution $\mathbf{x}_0$ with batch size $b$.
      \STATE Perturb $\mathbf{x}_0$ by forward transition $p(\mathbf{x}_t |\mathbf{x}_0)$.
      \STATE $\mathcal{L}_{\text{guidance}}(\boldsymbol{\psi}) =\frac{1}{b}\sum\limits_{\mathbf{x}_0, \mathbf{x}_t, t}\left[\left\|h_{\boldsymbol{\psi}}\left(\mathbf{x}_t, t\right)-(1-c_{\boldsymbol{\omega}}(\mathbf{x}_0))/c_{\boldsymbol{\omega}}(\mathbf{x}_0)\right\|_2^2\right]$
      \STATE Sample mini-batch data from target distribution $\mathbf{x}'_0$ with batch size $b$.
       \STATE Sample time $t \sim \operatorname{Uniform}(\{1, \ldots, T\})$ and perturb $\mathbf{x}'_0$ by forward transition $q(\mathbf{x}'_t |\mathbf{x}'_0)$.
      \STATE $\mathcal{L}_{\text{cycle}}=\frac{1}{b}\sum\limits_{\mathbf{x}'_0, \mathbf{x}'_t, t}\left[\left\|h_{\boldsymbol{\psi}}\left(\mathbf{x}'_t, t\right)-c'_{\boldsymbol{\omega}}(\mathbf{x}'_0, t)\right\|_2^2\right].$
      \STATE $\mathcal{L}_{\text{consistence}}= \frac{1}{b}\sum\limits_{\mathbf{x}'_0, \mathbf{x}'_t, t}\left[\left\|s_{\text{source}}(\mathbf{x}'_t, t)+ \nabla_{\mathbf{x}'_t} \log h_{\boldsymbol{\psi}}\left(\mathbf{x}'_t, t\right)-
\nabla_{\mathbf{x}'_t} \log q(\mathbf{x}_t | \mathbf{x}'_0)\right\|_2^2\right].$
    \STATE Take gradient descent step on
$$
\nabla_{\boldsymbol{\psi}}\left\{\mathcal{L}_{\text{guidance}}+ \eta_1 \ \mathcal{L}_{\text{cycle}} + \eta_2 \ \mathcal{L}_{\text{consistence}} \right\}.
$$
\UNTIL{converged.}
    \STATE \textbf{return} weights of guidance network  $\boldsymbol{\psi}$.
	\end{algorithmic}
\end{algorithm}

\subsection{Ablation Studies on simulations} \label{ap:regu}
In Figure \ref{fig:abl}, we demonstrate the ablation studies on simulations. We can see that only using the consistency regularization term (Figure \ref{fig:abl} (b)) is not able to recover the true distribution in the target domain. Our guidance loss together with cycle regularization can learn a good approximation of target distribution while adding consistency regularization can achieve better performance.
\begin{figure}[htb]
\label{fig:abl}
\centering
\subfigure[Source diffusion]{\includegraphics[width=0.28\textwidth]{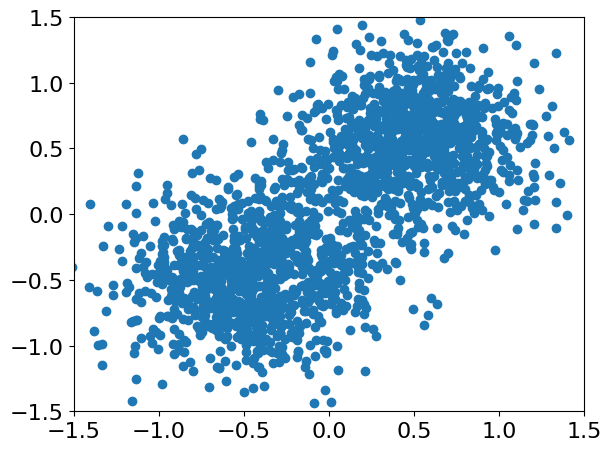}}
\subfigure[Only $\mathcal{L}_{\text{guidance}}$]{\includegraphics[width=0.28\textwidth]{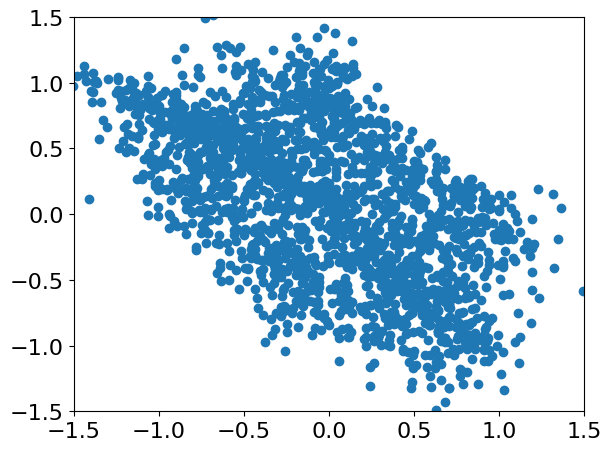}}
\subfigure[$\mathcal{L}_{\text{guidance}}+\mathcal{L}_{\text{cycle}}$]{\includegraphics[width=0.28\textwidth]{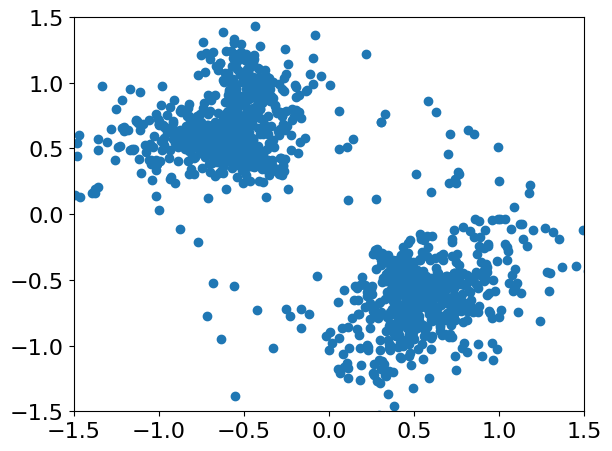}}
\subfigure[$\mathcal{L}_{\text{guidance}}+\mathcal{L}_{\text{consistence}}$]{\includegraphics[width=0.28\textwidth]{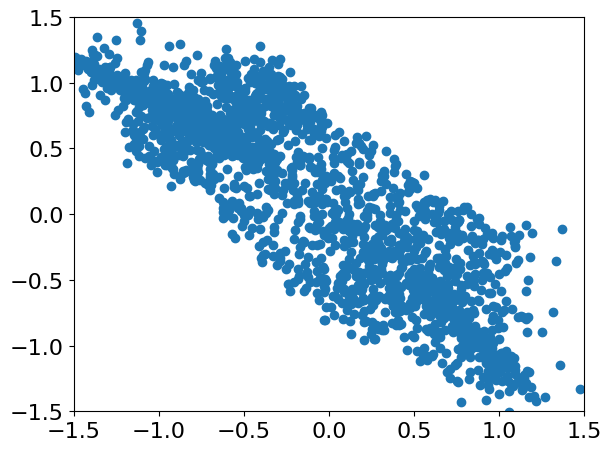}}
\subfigure[$\mathcal{L}_{\text{guidance}}+\mathcal{L}_{\text{cycle}}+\mathcal{L}_{\text{consistence}}$]{\includegraphics[width=0.28\textwidth]{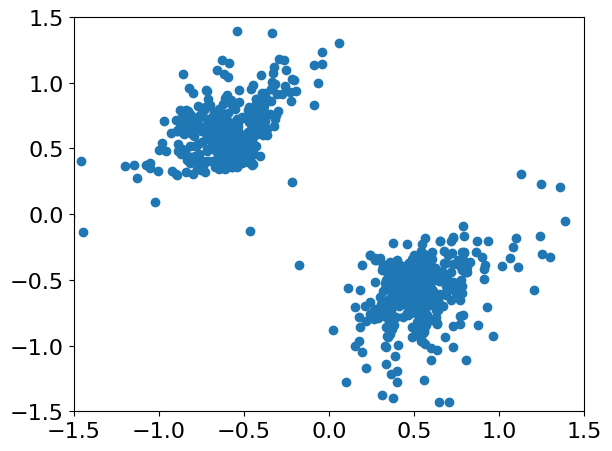}}
\subfigure[Target]{\includegraphics[width=0.28\textwidth]{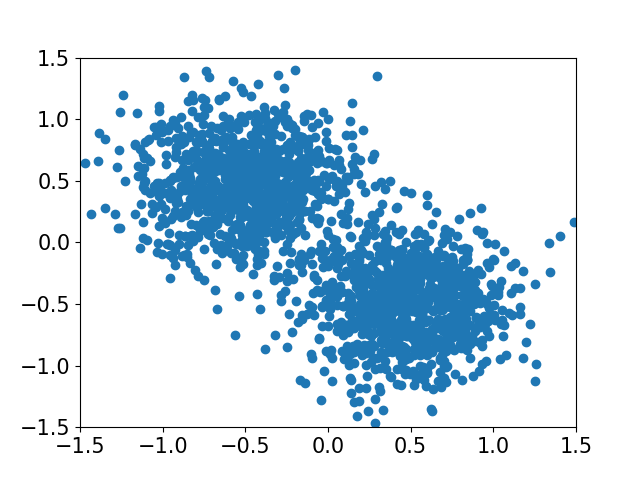}}
\caption{An illustration of the effectiveness of cycle regularization and consistency regularization proposed in Section \ref{sec:add_reg}.}
\end{figure}

\subsection{Implementation details for ECG Benchmark}
\label{ap:imple_detail}
For TGDP and all of the baseline methods, we utilize the same architecture as the conditional generative models for ECG data, SSSM-ECG \cite{LopezAlcaraz2023DiffusionbasedCE}. For {\it Vanilla Diffusion}, we train the diffusion model for 100k iterations by Adam optimizer with a learning rate 2$\mathrm{e}^{-4}$. For {\it Finetune Generator}, we finetune the pre-trained diffusion model for 50k iterations by Adam optimizer with a learning rate 2$\mathrm{e}^{-5}$. For TGDP, we adopt a 4-layer MLP with 512 hidden units and SiLU activation function as the backbone of the guidance network. We train the guidance network for 50k iterations by Adam optimizer with a learning rate 2$\mathrm{e}^{-4}$. For utility evaluation, we adopt the same architecture, xresnet1d50 \cite{Strodthoff2020DeepLF}, as the backbone. We train the classifier from sketch for 50 epochs with with a learning rate 1e-2. For Finetune Classifier, we finetune a pre-trained classifier for 30 epochs with with a learning rate 1e-3.

\section*{NeurIPS Paper Checklist}
\begin{enumerate}

\item {\bf Claims}
    \item[] Question: Do the main claims made in the abstract and introduction accurately reflect the paper's contributions and scope?
    \item[] \answerYes{}.
    \item[] Justification: We summarize the contributions and scope in Abstract as "we prove the optimal diffusion model for the target domain integrates pre-trained diffusion models with additional guidance" and we also summarize the main contribution of our paper in Introduction.
     \item[] Guidelines:
    \begin{itemize}
        \item The answer NA means that the abstract and introduction do not include the claims made in the paper.
        \item The abstract and/or introduction should clearly state the claims made, including the contributions made in the paper and important assumptions and limitations. A No or NA answer to this question will not be perceived well by the reviewers. 
        \item The claims made should match theoretical and experimental results, and reflect how much the results can be expected to generalize to other settings. 
        \item It is fine to include aspirational goals as motivation as long as it is clear that these goals are not attained by the paper. 
    \end{itemize}

\item {\bf Limitations}
    \item[] Question: Does the paper discuss the limitations of the work performed by the authors?
    \item[] Answer: \answerYes{}.
    \item[] Justification: We summarize the limitations of this paper in Section \ref{sec:conclusion}.
    \item[] Guidelines:
    \begin{itemize}
        \item The answer NA means that the paper has no limitation while the answer No means that the paper has limitations, but those are not discussed in the paper. 
        \item The authors are encouraged to create a separate "Limitations" section in their paper.
        \item The paper should point out any strong assumptions and how robust the results are to violations of these assumptions (e.g., independence assumptions, noiseless settings, model well-specification, asymptotic approximations only holding locally). The authors should reflect on how these assumptions might be violated in practice and what the implications would be.
        \item The authors should reflect on the scope of the claims made, e.g., if the approach was only tested on a few datasets or with a few runs. In general, empirical results often depend on implicit assumptions, which should be articulated.
        \item The authors should reflect on the factors that influence the performance of the approach. For example, a facial recognition algorithm may perform poorly when image resolution is low or images are taken in low lighting. Or a speech-to-text system might not be used reliably to provide closed captions for online lectures because it fails to handle technical jargon.
        \item The authors should discuss the computational efficiency of the proposed algorithms and how they scale with dataset size.
        \item If applicable, the authors should discuss possible limitations of their approach to address problems of privacy and fairness.
        \item While the authors might fear that complete honesty about limitations might be used by reviewers as grounds for rejection, a worse outcome might be that reviewers discover limitations that aren't acknowledged in the paper. The authors should use their best judgment and recognize that individual actions in favor of transparency play an important role in developing norms that preserve the integrity of the community. Reviewers will be specifically instructed to not penalize honesty concerning limitations.
    \end{itemize}

\item {\bf Theory Assumptions and Proofs}
    \item[] Question: For each theoretical result, does the paper provide the full set of assumptions and a complete (and correct) proof?
    \item[] Answer: \answerYes{}.
    \item[] Justification: We carefully state the assumptions in the statement of all theorems in this paper and all of the theorems are proven in the Appendix.
     \item[] Guidelines:
    \begin{itemize}
        \item The answer NA means that the paper does not include theoretical results. 
        \item All the theorems, formulas, and proofs in the paper should be numbered and cross-referenced.
        \item All assumptions should be clearly stated or referenced in the statement of any theorems.
        \item The proofs can either appear in the main paper or the supplemental material, but if they appear in the supplemental material, the authors are encouraged to provide a short proof sketch to provide intuition. 
        \item Inversely, any informal proof provided in the core of the paper should be complemented by formal proofs provided in appendix or supplemental material.
        \item Theorems and Lemmas that the proof relies upon should be properly referenced. 
    \end{itemize}

    \item {\bf Experimental Result Reproducibility}
    \item[] Question: Does the paper fully disclose all the information needed to reproduce the main experimental results of the paper to the extent that it affects the main claims and/or conclusions of the paper (regardless of whether the code and data are provided or not)?
    \item[] Answer: \answerYes{}.
    \item[] Justification: We provide all of the details to reproduce the experiments in this paper and we also release our code. 
    \item[] Guidelines:
    \begin{itemize}
        \item The answer NA means that the paper does not include experiments.
        \item If the paper includes experiments, a No answer to this question will not be perceived well by the reviewers: Making the paper reproducible is important, regardless of whether the code and data are provided or not.
        \item If the contribution is a dataset and/or model, the authors should describe the steps taken to make their results reproducible or verifiable. 
        \item Depending on the contribution, reproducibility can be accomplished in various ways. For example, if the contribution is a novel architecture, describing the architecture fully might suffice, or if the contribution is a specific model and empirical evaluation, it may be necessary to either make it possible for others to replicate the model with the same dataset, or provide access to the model. In general. releasing code and data is often one good way to accomplish this, but reproducibility can also be provided via detailed instructions for how to replicate the results, access to a hosted model (e.g., in the case of a large language model), releasing of a model checkpoint, or other means that are appropriate to the research performed.
        \item While NeurIPS does not require releasing code, the conference does require all submissions to provide some reasonable avenue for reproducibility, which may depend on the nature of the contribution. For example
        \begin{enumerate}
            \item If the contribution is primarily a new algorithm, the paper should make it clear how to reproduce that algorithm.
            \item If the contribution is primarily a new model architecture, the paper should describe the architecture clearly and fully.
            \item If the contribution is a new model (e.g., a large language model), then there should either be a way to access this model for reproducing the results or a way to reproduce the model (e.g., with an open-source dataset or instructions for how to construct the dataset).
            \item We recognize that reproducibility may be tricky in some cases, in which case authors are welcome to describe the particular way they provide for reproducibility. In the case of closed-source models, it may be that access to the model is limited in some way (e.g., to registered users), but it should be possible for other researchers to have some path to reproducing or verifying the results.
        \end{enumerate}
    \end{itemize}

\item {\bf Open access to data and code}
    \item[] Question: Does the paper provide open access to the data and code, with sufficient instructions to faithfully reproduce the main experimental results, as described in supplemental material?
    \item[] Answer: \answerYes{}.
    \item[] Justification: We provide the code for all of the experiments together with clear instructions.
        \item[] Guidelines:
    \begin{itemize}
        \item The answer NA means that paper does not include experiments requiring code.
        \item Please see the NeurIPS code and data submission guidelines (\url{https://nips.cc/public/guides/CodeSubmissionPolicy}) for more details.
        \item While we encourage the release of code and data, we understand that this might not be possible, so “No” is an acceptable answer. Papers cannot be rejected simply for not including code, unless this is central to the contribution (e.g., for a new open-source benchmark).
        \item The instructions should contain the exact command and environment needed to run to reproduce the results. See the NeurIPS code and data submission guidelines (\url{https://nips.cc/public/guides/CodeSubmissionPolicy}) for more details.
        \item The authors should provide instructions on data access and preparation, including how to access the raw data, preprocessed data, intermediate data, and generated data, etc.
        \item The authors should provide scripts to reproduce all experimental results for the new proposed method and baselines. If only a subset of experiments are reproducible, they should state which ones are omitted from the script and why.
        \item At submission time, to preserve anonymity, the authors should release anonymized versions (if applicable).
        \item Providing as much information as possible in supplemental material (appended to the paper) is recommended, but including URLs to data and code is permitted.
    \end{itemize}

\item {\bf Experimental Setting/Details}
    \item[] Question: Does the paper specify all the training and test details (e.g., data splits, hyperparameters, how they were chosen, type of optimizer, etc.) necessary to understand the results?
    \item[] Answer: \answerYes{}.
    \item[] Justification: Experimental Setting/Details can be found in our paper and code.
    \item[] Guidelines:
    \begin{itemize}
        \item The answer NA means that the paper does not include experiments.
        \item The experimental setting should be presented in the core of the paper to a level of detail that is necessary to appreciate the results and make sense of them.
        \item The full details can be provided either with the code, in appendix, or as supplemental material.
    \end{itemize}

\item {\bf Experiment Statistical Significance}
    \item[] Question: Does the paper report error bars suitably and correctly defined or other appropriate information about the statistical significance of the experiments?
    \item[] Answer: \answerYes{}.
    \item[] Justification: The standard deviation of the experiments in ECG Benchmark has been provided in Table \ref{tab:utility_newest}.
        \item[] Guidelines:
    \begin{itemize}
        \item The answer NA means that the paper does not include experiments.
        \item The authors should answer "Yes" if the results are accompanied by error bars, confidence intervals, or statistical significance tests, at least for the experiments that support the main claims of the paper.
        \item The factors of variability that the error bars are capturing should be clearly stated (for example, train/test split, initialization, random drawing of some parameter, or overall run with given experimental conditions).
        \item The method for calculating the error bars should be explained (closed form formula, call to a library function, bootstrap, etc.)
        \item The assumptions made should be given (e.g., Normally distributed errors).
        \item It should be clear whether the error bar is the standard deviation or the standard error of the mean.
        \item It is OK to report 1-sigma error bars, but one should state it. The authors should preferably report a 2-sigma error bar than state that they have a 96\% CI, if the hypothesis of Normality of errors is not verified.
        \item For asymmetric distributions, the authors should be careful not to show in tables or figures symmetric error bars that would yield results that are out of range (e.g. negative error rates).
        \item If error bars are reported in tables or plots, The authors should explain in the text how they were calculated and reference the corresponding figures or tables in the text.
    \end{itemize}

\item {\bf Experiments Compute Resources}
    \item[] Question: For each experiment, does the paper provide sufficient information on the computer resources (type of compute workers, memory, time of execution) needed to reproduce the experiments?
    \item[] Answer: \answerYes{}.
    \item[] Justification: We utilized a computing cluster equipped with 6 NVIDIA GeForce 3090 GPUs with memory 24268MiB and Intel(R) Xeon(R) Platinum 8352Y CPUs @ 2.20GHz. The computational costs for each of the individual experimental runs can be found in the following table. The total computational time of all experiments in this paper is around 200 GPU hours.

    \begin{table}[htbp]
\centering

\begin{tabular}{c|c|c}
\toprule  Experiments & Memory-Usage & Running Time  \\
\midrule
Experiments in Table \ref{tab:likelihood}   &  1281MiB &  5min \\ 
\midrule
Training Vanilla Diffusion  & 19815MiB & 1h  \\  %
Finetune Generator & 19815MiB &  40min \\  %
Training TGDP & 19597MiB &  40min \\  %
Sampling & 9535MiB &  21h \\  %
\midrule
Experiments in Table \ref{tab:utility_newest}&  6075MB & 10min \\  %

\bottomrule
\end{tabular}
\end{table}

    \item[] Guidelines:
    \begin{itemize}
        \item The answer NA means that the paper does not include experiments.
        \item The paper should indicate the type of compute workers CPU or GPU, internal cluster, or cloud provider, including relevant memory and storage.
        \item The paper should provide the amount of compute required for each of the individual experimental runs as well as estimate the total compute. 
        \item The paper should disclose whether the full research project required more compute than the experiments reported in the paper (e.g., preliminary or failed experiments that didn't make it into the paper). 
    \end{itemize}
    
\item {\bf Code Of Ethics}
    \item[] Question: Does the research conducted in the paper conform, in every respect, with the NeurIPS Code of Ethics \url{https://neurips.cc/public/EthicsGuidelines}?
    \item[] Answer: \answerYes{}.
    \item[] Justification: We review the NeurIPS Code of Ethics and ensure our compliance with its requirements.
    \item[] Guidelines:
    \begin{itemize}
        \item The answer NA means that the authors have not reviewed the NeurIPS Code of Ethics.
        \item If the authors answer No, they should explain the special circumstances that require a deviation from the Code of Ethics.
        \item The authors should make sure to preserve anonymity (e.g., if there is a special consideration due to laws or regulations in their jurisdiction).
    \end{itemize}

\item {\bf Broader Impacts}
    \item[] Question: Does the paper discuss both potential positive societal impacts and negative societal impacts of the work performed?
    \item[] Answer: \answerYes{}.
    \item[] Justification: We summarize the potential negative societal impacts of this paper together with the corresponding solutions to mitigate the negative impacts in Section \ref{sec:conclusion}.
    \item[] Guidelines:
    \begin{itemize}
        \item The answer NA means that there is no societal impact of the work performed.
        \item If the authors answer NA or No, they should explain why their work has no societal impact or why the paper does not address societal impact.
        \item Examples of negative societal impacts include potential malicious or unintended uses (e.g., disinformation, generating fake profiles, surveillance), fairness considerations (e.g., deployment of technologies that could make decisions that unfairly impact specific groups), privacy considerations, and security considerations.
        \item The conference expects that many papers will be foundational research and not tied to particular applications, let alone deployments. However, if there is a direct path to any negative applications, the authors should point it out. For example, it is legitimate to point out that an improvement in the quality of generative models could be used to generate deepfakes for disinformation. On the other hand, it is not needed to point out that a generic algorithm for optimizing neural networks could enable people to train models that generate Deepfakes faster.
        \item The authors should consider possible harms that could arise when the technology is being used as intended and functioning correctly, harms that could arise when the technology is being used as intended but gives incorrect results, and harms following from (intentional or unintentional) misuse of the technology.
        \item If there are negative societal impacts, the authors could also discuss possible mitigation strategies (e.g., gated release of models, providing defenses in addition to attacks, mechanisms for monitoring misuse, mechanisms to monitor how a system learns from feedback over time, improving the efficiency and accessibility of ML).
    \end{itemize}
    
\item {\bf Safeguards}
    \item[] Question: Does the paper describe safeguards that have been put in place for responsible release of data or models that have a high risk for misuse (e.g., pretrained language models, image generators, or scraped datasets)?
    \item[] Answer: \answerNA{}.
    \item[] Justification: The main contribution of our paper is an efficient way for transferring a pre-trained model on target distribution. We do not release data or models that  have a high risk for misuse.
     \item[] Guidelines:
    \begin{itemize}
        \item The answer NA means that the paper poses no such risks.
        \item Released models that have a high risk for misuse or dual-use should be released with necessary safeguards to allow for controlled use of the model, for example by requiring that users adhere to usage guidelines or restrictions to access the model or implementing safety filters. 
        \item Datasets that have been scraped from the Internet could pose safety risks. The authors should describe how they avoided releasing unsafe images.
        \item We recognize that providing effective safeguards is challenging, and many papers do not require this, but we encourage authors to take this into account and make a best faith effort.
    \end{itemize}

\item {\bf Licenses for existing assets}
    \item[] Question: Are the creators or original owners of assets (e.g., code, data, models), used in the paper, properly credited and are the license and terms of use explicitly mentioned and properly respected?
    \item[] Answer: \answerYes{}.
    \item[] Justification: The license of the dataset and code used in this paper can be found in the following table.
    
\begin{table}[htbp]
\centering
 \resizebox{\textwidth}{!}{
 \begin{tabular}{c|c|c}
\toprule
\textbf{Assets} & \textbf{License} & \textbf{Link} \\ \midrule
PTB-XL & CC-BY 4.0 & \url{https://physionet.org/content/ptb-xl/1.0.3/} \\ \midrule
ICBEB2018  & CC0: Public Domain & \url{https://www.kaggle.com/datasets/bjoernjostein/china-12lead-ecg-challenge-database} \\ \midrule
SSSD-ECG & MIT License & \url{https://github.com/AI4HealthUOL/SSSD-ECG?tab=readme-ov-file} \\ \midrule
ECG Benchmarks & CC-BY 4.0 & \url{https://github.com/helme/ecg_ptbxl_benchmarking} \\ \bottomrule
\end{tabular}}
\end{table}
    \item[] Guidelines:
    \begin{itemize}
        \item The answer NA means that the paper does not use existing assets.
        \item The authors should cite the original paper that produced the code package or dataset.
        \item The authors should state which version of the asset is used and, if possible, include a URL.
        \item The name of the license (e.g., CC-BY 4.0) should be included for each asset.
        \item For scraped data from a particular source (e.g., website), the copyright and terms of service of that source should be provided.
        \item If assets are released, the license, copyright information, and terms of use in the package should be provided. For popular datasets, \url{paperswithcode.com/datasets} has curated licenses for some datasets. Their licensing guide can help determine the license of a dataset.
        \item For existing datasets that are re-packaged, both the original license and the license of the derived asset (if it has changed) should be provided.
        \item If this information is not available online, the authors are encouraged to reach out to the asset's creators.
    \end{itemize}

\item {\bf New Assets}
    \item[] Question: Are new assets introduced in the paper well documented and is the documentation provided alongside the assets?
    \item[] Answer: \answerYes{}.
    \item[] Justification: The code provided in the supplementary material follows the CC-BY 4.0 license.
    \item[] Guidelines:
    \begin{itemize}
        \item The answer NA means that the paper does not release new assets.
        \item Researchers should communicate the details of the dataset/code/model as part of their submissions via structured templates. This includes details about training, license, limitations, etc. 
        \item The paper should discuss whether and how consent was obtained from people whose asset is used.
        \item At submission time, remember to anonymize your assets (if applicable). You can either create an anonymized URL or include an anonymized zip file.
    \end{itemize}

\item {\bf Crowdsourcing and Research with Human Subjects}
    \item[] Question: For crowdsourcing experiments and research with human subjects, does the paper include the full text of instructions given to participants and screenshots, if applicable, as well as details about compensation (if any)? 
    \item[] Answer: \answerNA{}.
    \item[] Justification: This paper does not involve crowdsourcing nor research with human subjects.
    \item[] Guidelines:
    \begin{itemize}
        \item The answer NA means that the paper does not involve crowdsourcing nor research with human subjects.
        \item Including this information in the supplemental material is fine, but if the main contribution of the paper involves human subjects, then as much detail as possible should be included in the main paper. 
        \item According to the NeurIPS Code of Ethics, workers involved in data collection, curation, or other labor should be paid at least the minimum wage in the country of the data collector. 
    \end{itemize}

\item {\bf Institutional Review Board (IRB) Approvals or Equivalent for Research with Human Subjects}
    \item[] Question: Does the paper describe potential risks incurred by study participants, whether such risks were disclosed to the subjects, and whether Institutional Review Board (IRB) approvals (or an equivalent approval/review based on the requirements of your country or institution) were obtained?
    \item[] Answer: \answerNA{}.
    \item[] Justification: This paper does not involve crowdsourcing nor research with human subjects.
    \item[] Guidelines:
    \begin{itemize}
        \item The answer NA means that the paper does not involve crowdsourcing nor research with human subjects.
        \item Depending on the country in which research is conducted, IRB approval (or equivalent) may be required for any human subjects research. If you obtained IRB approval, you should clearly state this in the paper. 
        \item We recognize that the procedures for this may vary significantly between institutions and locations, and we expect authors to adhere to the NeurIPS Code of Ethics and the guidelines for their institution. 
        \item For initial submissions, do not include any information that would break anonymity (if applicable), such as the institution conducting the review.
    \end{itemize}
\end{enumerate}



\end{document}